\documentclass[lettersize,journal]{IEEEtran}
\usepackage{amsmath,amsfonts}
\usepackage{algpseudocode}
\usepackage{algorithm}
\usepackage{array}
\usepackage[caption=false,font=normalsize,labelfont=sf,textfont=sf]{subfig}
\usepackage{textcomp}
\usepackage{stfloats}
\usepackage{url}
\usepackage{verbatim}
\usepackage{graphicx}
\usepackage{cite}
\usepackage{hyperref}
\usepackage{mathtools}
\usepackage{amsthm}
\newtheorem{theorem}{Theorem}[section]
\newtheorem{lemma}[theorem]{Lemma}

\theoremstyle{definition}

\theoremstyle{remark}

\makeatletter
\newcounter{phase}[algorithm]
\newlength{\phaserulewidth}
\newcommand{\setphaserulewidth}{\setlength{\phaserulewidth}}
\newcommand{\phase}[1]{%
  \vspace{-1.25ex}
  \Statex\leavevmode\llap{\rule{\dimexpr\labelwidth+\labelsep}{\phaserulewidth}}\vspace{-0.3ex}\rule{\linewidth}{\phaserulewidth}
  \Statex\strut\refstepcounter{phase}\vspace{-0.3ex}\textit{Part~\thephase~--~#1}
  \vspace{-1.25ex}\Statex\leavevmode\llap{\rule{\dimexpr\labelwidth+\labelsep}{\phaserulewidth}}\rule{\linewidth}{\phaserulewidth}}
\makeatother
\setphaserulewidth{.7pt}

\usepackage{graphicx}
\usepackage{setspace}
\usepackage{multicol}

\usepackage{makecell}
\usepackage{threeparttable} 

\begin{document}

\title{NN2Poly: A polynomial representation for deep feed-forward artificial neural networks}

\author{Pablo Morala, J.Alexandra Cifuentes, Rosa E. Lillo, Iñaki Ucar
\thanks{Manuscript received July 2022. This research is part of the I+D+i projects PID2019-104901RB-I00, PID2019-106811GB-C32 and PDC2022-133359-I00 founded by MCIN/ AEI/10.13039/501100011033/. (Corresponding author: Pablo Morala)}
\thanks{Pablo Morala, Rosa E. Lillo and Iñaki Ucar are with the Department of Statistics and the uc3m-Santander Big Data Institute, Universidad Carlos III de Madrid, Calle Madrid 126, 28903 Getafe (Madrid), Spain. E-mail: \url{pablo.morala@uc3m.es}; \url{rosaelvira.lillo@uc3m.es}; \url{inaki.ucar@uc3m.es}.}
\thanks{J. Alexandra Cifuentes is with ICADE, Department of Quantitative Methods, Faculty of Economics and Business Administration,  and the Institute for Research in Technology (IIT), ICAI School of Engineering,  Universidad Pontificia Comillas, Calle de Alberto Aguilera 23, 28015 Madrid, Spain. E-mail: \url{jacifuentes@comillas.edu}}}



\maketitle

\begin{abstract}
    Interpretability of neural networks and their underlying theoretical behavior remain an open field of study even after the great success of their practical applications, particularly with the emergence of deep learning. In this work, NN2Poly is proposed: a theoretical approach to obtain an explicit polynomial model that provides an accurate representation of an already trained fully-connected feed-forward artificial neural network (a multilayer perceptron or MLP). This approach extends a previous idea proposed in the literature, which was limited to single hidden layer networks, to work with arbitrarily deep MLPs in both regression and classification tasks. NN2Poly uses a Taylor expansion on the activation function, at each layer, and then applies several combinatorial properties to calculate the coefficients of the desired polynomials. Discussion is presented on the main computational challenges of this method, and the way to overcome them by imposing certain constraints during the training phase. Finally, simulation experiments as well as applications to real tabular data sets are presented to demonstrate the effectiveness of the proposed method.
\end{abstract}


\section{Introduction and related work}

While early neural networks were conceived several decades ago, the rise of deep learning has occurred mainly in the last decade \cite{lecunDeepLearning2015}. With the rise in the use of neural networks in their various forms, some research lines have been focused on solving their issues and limitations. Among some of the most addressed topics, the study of the interpretability and explainability of their behavior---aiming to solve their black-box nature \cite{benitezAreArtificialNeural1997, shwartz-zivOpeningBlackBox2017}---and the dimensioning of their structure, topology or hyperparameters---which is mainly done heuristically \cite{HIROSE199161, weymaereInitializationOptimizationMultilayer1994 , bengioPracticalRecommendationsGradientBased2012}---have attracted an increasing interest. Furthermore, it has been argued that the theoretical understanding of deep neural networks has not been developed at the same pace as their success in practical applications, resulting in the lack of a theory that explains their behavior satisfactorily~\cite{saxeInformationBottleneckTheory2019}.

In order to solve some of these problems, and to obtain a better understanding of the underlying aspects of these models, different efforts have been made in order to relate neural networks and deep learning models with more classical statistical tools. For example, the equivalence between kernel machines and neural networks has been explored \cite{leeDeepNeuralNetworks2018}, especially with the concept of Neural Tangent Kernel (NTK) \cite{jacotNeuralTangentKernel2018}, as well as their relation with general additive models \cite{agarwalNeuralAdditiveModels2021}. In an additional approach, LassoNet \cite{lemhadriLassoNetNeuralNetworks2021} integrates regularization with parameter learning to achieve feature sparsity.

Following this notion of joining traditional models with neural networks, the idea of using polynomials to explain the output of a related neural network has also been explored previously. In the work developed by Chen and Manry\cite{chenAnalysisDesignMultiLayer1998, chenBackpropagationRepresentationTheorem1990, chenConventionalModelingMultilayer1993}, polynomial basis functions are used to construct a relationship with a multilayer perceptron (MLP). It was presented in the context of the development of some of the universal approximation theorems in the 1990s\cite{cybenkoApproximationSuperpositionsSigmoidal1989, hornikMultilayerFeedforwardNetworks1989, hornikApproximationCapabilitiesMultilayer1991}, presenting an alternative to those theorems by using power series to obtain polynomials at each neuron, and then studying the process to obtain each term of the final polynomial through certain subnets from the original neural network. However, the coefficients of those power series are obtained by solving an optimization problem in a Least-Mean-Square approach.

Recently, the inherently explainable nature of polynomials as well as generative additive models (GAMs) has driven an increasing research interest in polynomial networks. In general, this kind of models benefits from deep learning ideas and methods directly in their training while keeping the interesting properties from polynomials or GAMs, such as their interpretability, as opposed to the opaqueness of neural network parameters. Furthermore, multiplicative interactions are a characteristic naturally present in polynomials, and have been proposed to have an important role in designing new neural network architectures \cite{jayakumarMultiplicativeInteractionsWhere2019}. Thus, for example, in \cite{chrysosDeepPolynomialNeural2022} Deep Polynomial Neural Networks (also named $\Pi$-networks) are presented as a model where the output is a high-order polynomial of the input, using high-order tensors obtained via tensor factorization. The computational cost of modeling all interactions is reduced with the use of factor sharing, following the ideas proposed in \cite{rendleFactorizationMachines2010}. Furthermore, the same authors of $\Pi$-networks had previously studied in \cite{chrysosPolyGANHighOrderPolynomial2019} similar ideas to present an alternative to Generative Adversarial Networks (GANs), PolyGAN, using high-order polynomials represented by high-order tensors. 

These lines of research have produced not only new ways and architectures to train polynomial-related neural networks, but also relevant studies about their  properties. In  a recent study \cite{fanExpressivityTrainabilityQuadratic2023},  the authors propose the use of quadratic networks as an alternative to traditional neural networks. These quadratic networks replace the standard inner product in the neurons with a quadratic function, allowing for an increased expressivity. To train these networks, the authors introduce a new strategy called ReLinear. Properties of polynomial neural networks are also explored in \cite{chorariaSpectralBiasPolynomial2022}, where spectral analysis is carried out to study the properties of Polynomial Neural Networks, with a particular focus on the relationship between spectral bias and the learning of higher frequencies. Finally, another recent development on this line is a new class of GAMs with tensor rank decompositions of polynomials \cite{dubeyScalableInterpretabilityPolynomials2022}, which claims to be able to model all higher-order feature interactions and defends the use of polynomial- and GAM-based solutions due to their interpretable nature.

This relation between neural networks and polynomials (in fact, polynomial regression) was also proposed in \cite{chengPolynomialRegressionAlternative2019}, where the authors conjecture that both models are equivalent. In this context, a method to obtain a polynomial representation from a given neural network with a single hidden layer and linear output was proposed in \cite{moralaMathematicalFrameworkInform2021}. This was the initial step where this paper relies on to propose NN2Poly, a theoretical approach that extends the previous work to arbitrarily deep fully-connected feed-forward neural networks, i.e., MLPs. In brief, the previous method from \cite{moralaMathematicalFrameworkInform2021} is used to obtain a polynomial representation at each unit in the first layer of an MLP, and from this point on, further polynomials are built iteratively, layer by layer, until the final output of the neural network is reached. The output is modeled as well by a final polynomial, in the regression case, or several ones in the classification case. Therefore, NN2Poly uses a direct combination of the weights of a given MLP, combined with a polynomial approximation of the activation function, and obtains an explicit expression for the polynomial coefficients, differing here from previous proposals.

The preceding research studies on polynomial neural networks have focused on the advantages of developing novel architectures and networks with a polynomial output that are specifically trained to achieve it. Our proposal distinguishes itself from these studies in that it can be employed on a pre-existing MLP to represent its internal parameters as a polynomial. As a result, our method can serve as an interpretability tool and provide insights into the patterns learned at each layer of the MLP. Particularly, this paper focuses on applications to tabular data sets, where the traditional statistical interpretation of polynomial coefficients is well-known and widely used.

Another relevant study to our proposal is the Analytically Modified Integral Transform Expansion (AMITE) \cite{sanchiricoAMITENovelPolynomial2021 }, which also proposes a polynomial expansion to analyze the non-linearities of neural networks. This innovative approach expands the activation functions using Chebyshev Polynomials. However, it is limited to a single hidden layer case, similar to \cite{moralaMathematicalFrameworkInform2021}, and hence, NN2Poly represents a significant advancement in the direction of generalizing it to deeper polynomials. This procedure constitutes a combinatorial problem, which means that the number of possible solutions grows exponentially with the number of features. This can limit the scalability of the method when the number of features is too high. To address this issue, we propose several solutions that can make the method more efficient and applicable to a wider range of problems.

The rest of the paper is organized as follows: Section \ref{section_notation} introduces some notation and definitions, and Section \ref{section_nn2poly_theoretical} presents the proposed method. Section \ref{section_practical_implementation} discusses practical implementation details and the considerations to overcome some of its main computational challenges. Finally, experimental results were obtained from simulations as well as from the application to two real tabular data sets.

\section{Notation and definitions}\label{section_notation}

\begin{table}[t]
    \centering
    \footnotesize
    \renewcommand{\arraystretch}{1.4}
    \caption{List of Symbols}
    \label{tab:symbols}
    \begin{tabular}{r|p{0.75\linewidth}}
    \hline
    \textbf{ Symbol} & \textbf{Description}\\
    \hline
    $L$ & Number of hidden layers $+ 1$.\\
    $\prescript{(l)}{}{y}_j$ & Output at layer $l$ and neuron $j$.\\
    $\prescript{(l)}{}{g}_j$ & Activation function at layer $l$.\\
    $\prescript{(l)}{}{w}_{ij}$ & Weight connecting neuron $i$ at layer $l-1$ with neuron $j$ at layer $l$.\\
    $x_i$ & Original variable $i$.\\
    $p$ & Dimension (number of variables).\\
    $q_l$ & Taylor expansion order at layer $l$.\\
    $Q$ & Polynomial order at layer $l$.\\
    $N_{p,Q}$ & Number of terms in a polynomial with $p$ variables and order $Q$.\\
    $\prescript{(l)}{\mathrm{in/out}}{P}_j$ & Input/Output polynomial at layer $l$ and neuron $j$.\\
    ${B}_{k}$ & Monomial $k$ inside a polyniamial.\\
    ${\beta}_{k}$ & Coefficient $k$ for the monomial $B_k$ inside a polynomial.\\
    $M$ & A multiset.\\
    $\mathcal{M}(p,Q)$ & set of all possible multisets $M$ for a given number of variables $p$ and total order $Q$.\\
    $\Vec{t}$ & Vector representation of the multiplicities of all variables in a  polynomial term / coefficient.\\
    $\pi(\Vec{t},Q,n)$ & Set of allowed vectors that satisfy the conditions from Lemma \ref{lemma_taylor_poly}, for the combinations of variables given by $\Vec{t}$, $Q$ and $n$.\\
    \hline
    \end{tabular}
\end{table}

\subsection{Neural Network}

Table \ref{tab:symbols} summarizes and defines the main symbols used throughout the paper. In  this work, we consider fully-connected feed-forward neural networks---also often referred to as Multilayer Perceptrons (MLPs)---with $L-1$ hidden layers and $h_{l}$ neurons at each layer $l$. The input variables to the network are denoted by $\Vec{x} = (x_1, \dots, x_p)$, with dimension $p$, and the output response is $\Vec{y} = (y_1, \dots, y_c)$, with $c=1$, if there is a single output (usual regression setting), and $c>1$ if there is more than one output (usual classification setting with $c$ classes). At any given layer $l$ and neuron $j$, its output is $\prescript{(l)}{}{y}_j$, while its inputs are the outputs from the previous layer, $\prescript{(l-1)}{}{y}_i$ for $i \in 1, \dots, h_{l-1}$. By definition, the input to the first hidden layer is the network's input, $\prescript{(0)}{}{y}_i = x_i$ for $i \in 1, \dots, p$. The network's final output is denoted by $\prescript{(L)}{}{y}_j$ for $i \in 1, \dots, c$. $\prescript{(l)}{}{W}$ denotes the weights matrix connecting layer $l-1$ to layer $l$, where its element at row $i$ and  column $j$ is denoted by $\prescript{(l)}{}{w}_{ij}$, and $\prescript{(l)}{}{g}$ is the activation function at that layer. Then, the output from each neuron $j = 1, \dots, h_l$ at layer $l$ can be written as follows:
\begin{equation}
    \prescript{(l)}{}{y}_j=\prescript{(l)}{}{g}\left(\prescript{(l)}{}{u}_j\right)=\prescript{(l)}{}{g}\left(\sum_{i=0}^{h_{l-1}}\prescript{(l)}{}{w}_{ij}\prescript{(l-1)}{}{y}_i\right),
    \label{eq_neuron_computation}
\end{equation}
where $\prescript{(l)}{}{u}_j$ is the synaptic potential, i.e., the value computed at the neuron before applying the activation function. Note also that the matrix $\prescript{(l)}{}{W}$ has dimensions $(h_{l-1} + 1)\times h_{l}$, including the bias term $\prescript{(l-1)}{}{y}_0 = 1$.

\subsection{Polynomials}

The main goal of the NN2Poly method described in Section \ref{section_nn2poly_theoretical} is to obtain a polynomial that approximately matches the predictions of a given pre-trained MLP. To this end, $P$ will denote a general polynomial of order $Q$ in $p$ variables, considering all possible interactions of variables up to this order $Q$. The cardinality of the polynomial terms can be computed as:
\begin{equation}
    N_{p,Q}=\sum_{T=0}^{Q}{{p+T-1}\choose{T}},
    \label{eq_M}
\end{equation}
where $T$ represents the order of each possible term or monomial, from $0$ to $Q$. The term ${{p+T-1}\choose{T}}$ represents the number of possible combinations when taking $T$ elements (the terms of the monomial) from a set with $p$ elements (the variables) with repetition where the order does not matter, i.e., the number of possible monomials in $p$ variables of order $T$. The vector of coefficients can then be defined as $\Vec{\beta}=(\beta_1, \dots, \beta_{N_{p,Q}})$, where $\beta_k$ is the coefficient associated to a specific monomial. Furthermore, each monomial will be denoted as $B_k$, which will contain $\beta_k$ and its associated combination of variables. Therefore, the polynomial can be represented explicitly as:
\begin{equation}
    P = \sum_{k=1}^{N_{p,Q}} B_k.
    \label{eq_poly_k}
\end{equation}

An alternative representation that is explicit about the variables considered in each monomial will be needed later. Each monomial of order $T$ can be represented by a vector $\Vec{t}=(t_1,t_2,\dots,t_p)$, where each element $t_i$ is an integer that represents the number of times that variable $i$ appears in the monomial, i.e., the multiplicity of that variable, for all $i=1,\dots,p$. Note that $T=\sum_{i=1}^{p} t_i$. As an example, the monomial containing the combination of variables $x_1^2 x_2 x_4$ in a polynomial with $p=4$ will be:
\begin{equation*}
    B_{(2,1,0,1)} = \beta_{(2,1,0,1)} \cdot x_1^2 x_2 x_4.
\end{equation*}
Using this representation, Eq. \ref{eq_poly_k} can be written using the multiplicities as:
\begin{equation}
    P = \sum_{\Vec{t} \in \mathcal{T}(p,Q)} B_{\Vec{t}} = \sum_{\Vec{t} \in \mathcal{T}(p,Q)} \beta_{\Vec{t}} \cdot x_{1}^{t_1}\dots x_{p}^{t_p},
    \label{eq_poly_vec_t}
\end{equation}
where $\mathcal{T}(p,Q)$ is the set of all possible vectors $\Vec{t}$, for a given number of variables $p$ and total order $Q$, and considering all interactions. Note that its cardinality is $|\mathcal{T}(p,Q)|=N_{p,Q}$. For convenience, the intercept will be denoted as $\beta_{(0)}$.

A bijection exists between the explicit $k$ index notation (Eq. \ref{eq_poly_k}) and the multiplicites  $\Vec{t}$ vector notation (Eq. \ref{eq_poly_vec_t}), and both will be used when developing the NN2Poly algorithm, according to the specific needs at each step. However, the explicit details of this bijection do not need to be defined to develop the algorithm.

\subsection{Multisets and multiset partitions}
\label{section_multiset_definition}

As NN2Poly will build polynomial terms based on sums and products of terms from previous polynomials, a combinatorial problem will arise when identifying the previous terms that are involved in the construction of an specific term of the new polynomial. In this context, the concepts of a multiset and its possible partitions have to be introduced.

A multiset $M$ is a set where its elements may appear more than once, i.e., they may have a multiplicity greater than 1. An example would be $M=\{1,1,2,3\}$, which is formed by elements $\{1\}$, $\{2\}$, $\{3\}$, with multiplicities $2$, $1$, $1$ respectively. 

A partition $V_r$ of a multiset $M$ (also known as multipartition) is a combination of $R_r$ multisets $V_{r,i}$ for $i=1,\dots, R_r$ such that its union forms the initial multiset. Following the previous example, the multiset $M=\{1,1,2,3\}$ can be partitioned into  $V_1 = \{\{1,2\},\{1,3\}\}$, which has $R_1=2$ elements, but also into $V_2 = \{\{1,1\},\{2\},\{3\}\}$, which has $R_2=3$ elements. As can be seen, there are many possible different partitions of a multiset. All possible partitions of the example $M=\{1,1,2,3\}$ are:
\begin{multicols}{2}
    \begin{itemize}
        \item $\{1,1,2,3\}$
        \item $\{1\},\{1,2,3\}$
        \item $\{2\},\{1,1,3\}$
        \item $\{3\},\{1,1,2\}$
        \item $\{1,1\},\{2,3\}$
        \item $\{1,2\},\{1,3\}$
        \item $\{1\},\{1\},\{2,3\}$
        \item $\{1\},\{2\},\{1,3\}$
        \item $\{1\},\{3\},\{1,2\}$
        \item $\{2\},\{3\},\{1,1\}$
        \item $\{1\},\{1\},\{2\},\{3\}$
    \end{itemize}
\end{multicols}

Obtaining all the possible partitions of a given multiset still presents some open problems and questions that remain unanswered \cite{knuthArtComputerProgramming2005}. Appendix \ref{appendix_multiset} contains the algorithm used to compute them, extracted from \cite{knuthArtComputerProgramming2005}.

Using these concepts, a third way of representing a given term in the polynomial can be defined by means of multisets. Using the $\Vec{t}$ vector notation (Eq. \ref{eq_poly_vec_t}), given the term $B_{\Vec{t}}$ where $\Vec{t}=(t_1,t_2,\dots,t_p)$, a multiset can be associated with it in the following manner: 
\begin{equation}
    M = \left\{\underbrace{1,\dots,1}_{t_1}, \underbrace{2,\dots,2}_{t_2}, \dots, \underbrace{p,\dots,p}_{t_p}\right\},
    \label{eq_multiset_M_from_t}
\end{equation}
where each variable $i$ appears $t_i$ times in the multiset. As an example, using the multiset notation, the monomial containing the combination of variables $x_1^2 x_2 x_4$ in a polynomial, with $p=4$, will be
\begin{equation*}
    B_{\{1,1,2,4\}} = \beta_{\{1,1,2,4\}} \cdot x_1^2 x_2 x_4.
\end{equation*}
Therefore, an alternative way of representing a polynomial as in Eq. \ref{eq_poly_k} and Eq. \ref{eq_poly_vec_t}, now using the multisets, is:
\begin{equation}
    P = \sum_{M \in \mathcal{M}(p,Q)} B_{M},
    \label{eq_poly_multiset}
\end{equation}
where $\mathcal{M}(p,Q)$ is the set of all possible multisets $M$ for a given number of variables $p$ and total order $Q$, considering all interactions. Note that its cardinality is, again, $|\mathcal{M}(p,Q)|=N_{p,Q}$. By definition, there exists a bijection that relates Eq. \ref{eq_poly_multiset} with Eq. \ref{eq_poly_vec_t}, and by the transitive property, also with Eq. \ref{eq_poly_k}. Therefore, the three representations of $P$ are equivalent and they will be used at different steps in Section \ref{section_nn2poly_theoretical}. Table \ref{tab:polynomial_notation} summarizes these three equivalent notations in order to get a look at the three representations of $P$.

\begin{table*}[t]
\centering
\footnotesize
\renewcommand{\arraystretch}{1.2}
\begin{threeparttable}
\caption{Polynomial notations.}
\label{tab:polynomial_notation}
\begin{tabular*}{0.85\textwidth}{r|l|l}
  \hline
  \textbf{Name (index)} & \textbf{Polynomial} & \textbf{Monomial example: $x_1^2 x_2 x_4$}  \\
  \hline
  \textbf{Explicit ($k$)} & 
  $P = \sum_{k=1}^{N_{p,Q}} B_k$
  & $ B_{k} = \beta_{k} \cdot x_1^2 x_2 x_4$, \\
  \textbf{Multiplicity ($\Vec{t}$)} & 
  $P = \sum_{\Vec{t} \in \mathcal{T}(p,Q)} B_{\Vec{t}} = \sum_{\Vec{t} \in \mathcal{T}(p,Q)} \beta_{\Vec{t}} \cdot x_{1}^{t_1}\dots x_{p}^{t_p}$
  & $B_{(2,1,0,1)} = \beta_{(2,1,0,1)} \cdot x_1^2 x_2^1 x_3^0x_4^1$ \\
  \textbf{Multiset ($M$)} & 
  $P = \sum_{M \in \mathcal{M}(p,Q)} B_{M}$
  & $ B_{1,1,2,4} = \beta_{1,1,2,4} \cdot x_1^2 x_2 x_4$ \\
  \hline
\end{tabular*}
\begin{tablenotes}\footnotesize
\item \emph{Note:} In the explicit notation, $k$ represents an explicit and unique index for each term in the given polynomial $P$, which would change depending on the number of variables and polynomial degree.
\end{tablenotes}
\end{threeparttable}
\end{table*}

\section{NN2Poly: Theoretical development}
\label{section_nn2poly_theoretical}

In this Section, the NN2Poly method is proposed and justified. An inductive argument will be employed, which will inspire the iterative implementation of the method discussed in Section \ref{section_practical_implementation}.

Consider a trained MLP with $L-1$ hidden layers where each neuron is computed as in Eq. \ref{eq_neuron_computation}. Then, the final objective of the NN2Poly method is to obtain a polynomial of the form presented in Eq. \ref{eq_poly_k}, or in its alternative representations, for each neuron at output layer $L$. In order to do so, an iterative approach will be employed, which will yield two polynomials at each neuron $j$ and layer $l$, one as the input to the non-linearity or activation function (denoted by the prefix ${\mathrm{in}}$), and a second one as the output (denoted by the prefix ${\mathrm{out}}$):
\begin{align}
    \prescript{(l)}{\mathrm{in}}{P}_j &=\sum_{\Vec{t} \in \mathcal{T}(p,Q)} \prescript{(l)}{\mathrm{in}}{B}_{j,\Vec{t}}\\ \prescript{(l)}{\mathrm{out}}{P}_j &=\sum_{\Vec{t} \in \mathcal{T}(p,Q)} \prescript{(l)}{\mathrm{out}}{B}_{j,\Vec{t}},
\end{align}
following the notation from Eq. \ref{eq_poly_vec_t}. Note that the two polynomials at the same layer will be related by the application of the activation function at such a layer, $\prescript{(l)}{}{g}()$, as follows:
\begin{equation}
    \prescript{(l)}{\mathrm{out}}{P}_j\approx \prescript{(l)}{}{g}\left(\prescript{(l)}{\mathrm{in}}{P}_j\right).
\end{equation}

In order to find this relationship, a Taylor expansion and several combinatorial properties will be exploited to obtain a polynomial in the original variables $x_1,\dots, x_p$ from the expansion of the non-linearity $g()$ applied to another polynomial in the same variables. This follows the idea used in \cite{moralaMathematicalFrameworkInform2021} to approximate the non-linearity of a single hidden layer NN, with the notable difference that the input to $g()$ in that case was a linear combination of the original variables instead of a higher order polynomial. Lemma \ref{lemma_taylor_poly} shows the way to obtain the coefficients of this polynomial. Note that, without loss of generality, the indices $l$ and $j$ denoting the layer and neuron can be omitted in the Lemma for the shake of simplicity.

\begin{lemma}
    Let $\prescript{}{\mathrm{in}}{P}$ be a polynomial of order $Q$ in $p$ variables and $g()$ a k-times differentiable function. Applying a Taylor expansion around $0$ and up to order $q$ yields a new polynomial $\prescript{}{\mathrm{out}}{P}$ in the same variables, of order $Q \times q$, whose coefficients are given by
    \begin{equation}
        \prescript{}{\mathrm{out}}{\beta}_{\vec{t}} = 
        \sum_{n=0}^{q}\dfrac{g^{(n)}(0)}{n!}
        \sum_{ \Vec{n} \in \pi(\Vec{t},Q,n)}
        \left(\begin{array}{c} n \\ \Vec{n} \end{array}\right)\prod_{k=1}^{N_{p,Q}} \prescript{}{\mathrm{in}}{\beta}_{k}^{n_{k}},
    \end{equation}
    for all $\Vec{t} \in \mathcal{T}(p,Q \times q)$, where $\pi(\Vec{t},Q,n)$ is the set of vectors $\Vec{n}=(n_1,\dots, n_{N_{p,Q}})$ that represents all the possible combinations of terms $\prescript{}{\mathrm{in}}{\beta}$ that are needed to obtain $\prescript{}{\mathrm{out}}{\beta}$, satisfying the following conditions:
    \begin{itemize}
    \item \textbf{Condition 1}: $n_{1}+\dots+n_{M}=n$. 
    \item \textbf{Condition 2}: For all $k$, the order $T_k$ of the monomial $\prescript{}{\mathrm{in}}{B}_k$ associated to $\prescript{}{\mathrm{in}}{\beta}_k$  must satisfy that $T_k  \le Q$.
\end{itemize}
    \label{lemma_taylor_poly}
\end{lemma}

\begin{proof}
Applying Taylor expansion to function $g$ around $0$ and up to order $q$:
\begin{align}
        g\left(\prescript{}{\mathrm{in}}{P}\right) &= g\left(\sum_{\Vec{t} \in \mathcal{T}(p,Q)} \prescript{}{\mathrm{in}}{B}_{\Vec{t}}\right) \nonumber\\ &\approx\sum_{n=0}^{q}\dfrac{g^{(n)}(0)}{n!}\left(\sum_{\Vec{t} \in \mathcal{T}(p,Q)} \prescript{}{\mathrm{in}}{B}_{\Vec{t}}\right)^n.
    \label{eq_taylor_proof}
\end{align}

Using the $k$ notation from Eq. \ref{eq_poly_k}, the previous equation can be rewritten as:
\begin{equation}
        g\left(\prescript{}{\mathrm{in}}{P}\right) \approx \sum_{n=0}^{q}\dfrac{g^{(n)}(0)}{n!}\left(\sum_{k=0}^{N_{p,Q}} \prescript{}{\mathrm{in}}{B}_{k}\right)^n,
\end{equation} where $N_{p,Q}$ is the total number of terms in the polynomial. Then, the last term can be expanded using the multinomial theorem as follows:
\begin{multline}
    \left(\sum_{k=0}^{N_{p,Q}} \prescript{}{\mathrm{in}}{B}_{k}\right)^n =
    \sum_{n_{1}+\dots+n_{N_{p,Q}}=n}
    \left(\begin{array}{c} n \\ \Vec{n} \end{array}\right) \prod_{k=1}^{N_{p,Q}}\prescript{}{\mathrm{in}}{B}_{k}^{n_k},
\end{multline}
where vector $\Vec{n} = (n_1,\dots, n_{N_{p,Q}})$ denotes a vector with non negative integer components, and the multinomial coefficient is:
\begin{equation}\label{eq_multinomial_coefficient}
    \left(\begin{array}{c} n \\ \Vec{n} \end{array}\right)
= \frac{n!}{n_{1}! n_{2}! \cdots n_{N_{p,Q}}!}.
\end{equation}
Note that the summation is computed over all vectors $\Vec{n}$, such that the sum of all its components is $n$. Then, the expansion of $g()$ is:
\begin{equation}
    g\left(\prescript{}{\mathrm{in}}{P}\right) =
    \sum_{n=0}^{q}\dfrac{g^{(n)}(0)}{n!}
    \sum_{n_{1}+\dots+n_{N_{p,Q}}=n}
    \left(\begin{array}{c} n \\ \Vec{n} \end{array}\right) \prod_{k=1}^{N_{p,Q}}\prescript{}{\mathrm{in}}{B}_{k}^{n_k}.
    \label{eq_multinomial_expansion_proof}
\end{equation}

Eq. \ref{eq_multinomial_expansion_proof} consists of sums and products of constants and terms $\prescript{}{\mathrm{in}}{B}_{k}^{n_k}$, which contain monomials of the original variables per Eq. \ref{eq_poly_k}, so it is a polynomial in the original variables. The order of the polynomial will be given by the highest order term, which will be obtained when $n_k=n=q$ for some $k$ such that $\prescript{}{\mathrm{in}}{B}_{k}$ is a monomial of order $Q$ from the original polynomial, which will yield a term of order $Q \times q$ in the output polynomial. 

To obtain the coefficients $\prescript{}{\mathrm{out}}{\beta}_{\vec{t}}$, the polynomial  has to be divided into terms associated with each combination of variables. For a given coefficient $\prescript{}{\mathrm{out}}{\beta}_{\vec{t}}$, all the products of terms $\prescript{}{\mathrm{in}}{B}$ that would yield the combination of variables determined by $\Vec{t}$ need to be considered. At this point, it is convenient to switch to the multiset notation $M$ from Eq. \ref{eq_poly_multiset}. As an illustrative example, consider the coefficient $\prescript{}{\mathrm{out}}{\beta}_{(2,1,0,1)}$ (vector $\Vec{t}$ notation), which is the same as $\prescript{}{\mathrm{out}}{\beta}_{\{1,1,2,4\}}$ (multiset $M$ notation), associated to the combination $x_1^2x_2x_4$. This combination of variables is included in the term $\prescript{}{\mathrm{in}}{B}_{\{1,1,2,4\}}$, but also in the product of terms $\prescript{}{\mathrm{in}}{B}_{\{1,1\}}\prescript{}{\mathrm{in}}{B}_{\{2,4\}}$ or in the product of the four terms $\prescript{}{\mathrm{in}}{B}_{\{1\}}^2 \prescript{}{\mathrm{in}}{B}_{\{2\}} \prescript{}{\mathrm{in}}{B}_{\{4\}}$. 
Finding all the possible combinations is equivalent to consider all possible partitions of the multiset $M$ associated with $\Vec{t}$. From the previous example, it is clear that $V_1 = \{\{1,1\},\{2,4\}\}$ and $V_2 = \{\{1\}, \{1\}, \{2\}, \{4\}\}$ are both possible partitions of the multiset $M=\{1,1,2,4\}$.

Then, a vector $\Vec{n}$ from Eq. \ref{eq_multinomial_expansion_proof} can be associated with each partition $V_r$ of the multiset $M$. This vector $\Vec{n}$ will contain the multiplicity of term $\prescript{}{\mathrm{in}}{B}_{V_{r,i}}$ at position $n_k$ where $k$ is the index (using notation from Eq. \ref{eq_poly_k}) associated with the multiset $V_{r,i}$ (using notation from Eq. \ref{eq_poly_multiset}) from partition $V_{r}$. This can be done for all $i=1,\cdots,R_{r}$, where $R_{r}$ is the number of elements (multisets) in partition $V_{r}$. Then, for each vector $\vec{n}$, the product $\prod_{k=1}^{N_{p,Q}}\prescript{}{\mathrm{in}}{B}_{k}^{n_k}$ will yield the desired combinations of original variables, as it will contain the product of terms determined by the chosen partition. However, not all partitions of the multiset are allowed, and therefore, vectors $\vec{n}$ are restricted by two different conditions:
\begin{itemize}
    \item \textbf{Condition 1}: $n_{1}+\dots+n_{M}=n$. This limits the number of elements in a partition to be $n$, as its the sum of the multiplicities $n_k$ contained in $\Vec{n}$. Note that $n$ goes from $0$ to $q$, the truncation order for the Taylor expansion. \textbf{Example}: consider the term $\prescript{}{\mathrm{out}}{B}_{\{1,1,2,3\}}$, with the associated multiset $\{1,1,2,3\}$; let $n=q=2$; then, partition $\{1\},\{1\},\{2\},\{3\}$ would not be allowed by this condition, as it contains 4 elements and $n_{1}+\dots+n_{M}= 4 \not= 2 = n$.
    
    \item \textbf{Condition 2}: For all $k$, the order $T_k$ of its associated monomial $\prescript{}{\mathrm{in}}{B}_k$ must satisfy that $T_k  \le Q$. This ensures that the order of the monomial, represented by each element in the partition, does not exceed the order $Q$ of the original polynomial. \textbf{Example}: let $Q=2$, $q=2$; the new polynomial $\prescript{}{\mathrm{out}}{P}$ will be of order $Q \times q = 4$, and it will contain the term $\prescript{}{\mathrm{out}}{B}_{\{1,1,2,3\}}$, with the associated multiset $\{1,1,2,3\}$; then partition $\{1\},\{1,2,3\}$ would not be allowed by this condition, as it contains the element $\{1,2,3\}$, associated with the monomial $\prescript{}{\mathrm{in}}{B}_{\{1,2,3\}}$ of order 3, which is not included in the polynomial $\prescript{}{\mathrm{in}}{P}$ of order $Q=2$.
\end{itemize}

The allowed set of vectors $\Vec{n}$ that represent the valid partitions of the multiset $M$ determined by $\Vec{t}$, given the values $n$ and $Q$ and satisfying both previous restrictions will be denoted as $\pi(\Vec{t},Q,n)$. Then, the expression for each term in $\prescript{}{\mathrm{out}}{P}$ is:
\begin{equation}
    \prescript{}{\mathrm{out}}{B}_{\vec{t}} = 
    \sum_{n=0}^{q}\dfrac{g^{(n)}(0)}{n!}
    \sum_{ \Vec{n} \in \pi(\Vec{t},Q,n)}
    \left(\begin{array}{c} n \\ \Vec{n} \end{array}\right)\prod_{k=1}^{N_{p,Q}} \prescript{}{\mathrm{in}}{B}_{k}^{n_{k}},
\end{equation}
which leads to the expression for the coefficients:
\begin{equation}
    \prescript{}{\mathrm{out}}{\beta}_{\vec{t}} = 
    \sum_{n=0}^{q}\dfrac{g^{(n)}(0)}{n!}
    \sum_{ \Vec{n} \in \pi(\Vec{t},Q,n)}
    \left(\begin{array}{c} n \\ \Vec{n} \end{array}\right)\prod_{k=1}^{N_{p,Q}} \prescript{}{\mathrm{in}}{\beta}_{k}^{n_{k}}.
    \label{eq_final_coeff_lemma}
\end{equation}

\end{proof}

After Lemma \ref{lemma_taylor_poly} has been evaluated, it can be used to define the NN2Poly algorithm in an iterative manner. For any given neuron $j$, in the first hidden layer ($l=1$), its activation potential is $\prescript{(1)}{}{u}_j = \sum_{i=0}^{p}\prescript{(1)}{}{w}_{ij}\prescript{(0)}{}{y}_{i} =  \sum_{i=0}^{p}\prescript{(1)}{}{w}_{ij}x_i$. It is clear that this is a polynomial of order 1, as it is a linear combination of the original variables, and thus it will be denoted by $\prescript{(1)}{\mathrm{in}}{P}_j$. Then, the non-linearity is applied to this activation potential:
\begin{equation*}
    \prescript{(1)}{}{y}_j=\prescript{(1)}{}{g}\left(\prescript{(1)}{}{u}_j\right)=\prescript{(1)}{}{g}\left(\prescript{(1)}{\mathrm{in}}{P}_j\right) \approx \prescript{(1)}{\mathrm{out}}{P}_j,
\end{equation*}
where Lemma \ref{lemma_taylor_poly} has been applied in the last step to obtain the explicit coefficients of a new output polynomial. Then, the activation potential at the next layer, $l=2$, is a linear combination of polynomials, which yields another polynomial: 
\begin{equation*}
    \prescript{(2)}{}{u}_j = \sum_{i=0}^{h_1}\prescript{(2)}{}{w}_{ij}\prescript{(1)}{}{y}_i \approx
    \sum_{i=0}^{h_1}\prescript{(2)}{}{w}_{ij}  \prescript{(1)}{\mathrm{out}}{P}_i =
    \prescript{(2)}{\mathrm{in}}{P}_j,
\end{equation*}
where $\prescript{(1)}{}{y}_i \approx \prescript{(1)}{\mathrm{out}}{P}_i$. This process can then be iterated over all layers until $l=L-1$, and, in the final layer $l=L$, the algorithm stops at $\prescript{(L)}{\mathrm{in}}{P}_j \approx \prescript{(L)}{}{u}_j$ if the MLP output is linear, or stops at $\prescript{(L)}{\mathrm{out}}{P}_j \approx \prescript{(L)}{}{g}\left(\prescript{(L)}{\mathrm{in}}{P}_j\right)$ if the output layer contains a non-linearity. It is worth noting that, in classification tasks, the final layer often includes a non-linear activation function, such as the sigmoid. An alternative approach involves training a NN with a linear output, which can then be transformed into a probability using a function like the softmax to make class predictions. In the latter scenario, a polynomial can be obtained to represent the linear output before the softmax operation. This polynomial coefficients can be interpreted in a similar manner to the classical interpretation of the coefficients in a logistic regression model (See Section \ref{section_covertype}). The final description of the theoretical version of NN2Poly is presented in Algorithm \ref{alg_NN2Poly}. The set of constraints required to achieve an efficient and scalable implementation while ensuring an accurate approximation of the NN are discussed in Section \ref{section_practical_implementation}, where Algorithm \ref{alg_NN2Poly_practical} is presented.

\begin{algorithm}[t]
\small
\caption{NN2Poly}\label{alg_NN2Poly}
\begin{algorithmic}[1]
\Require Weight matrices $\prescript{(l)}{}{W}$, activation functions $\prescript{(l)}{}{g}$,  Taylor truncation order at each layer $q_l$.
\State Set $\prescript{(1)}{\mathrm{in}}{P}_j = \prescript{(1)}{}{u}_j$ for all $j=1, \dots, h_1$.
\For{$l=1, \cdots, L-1$}
    \For{$j=1, \cdots, h_l$} 
        \State \textbf{Compute} $\prescript{(l)}{\mathrm{out}}{P}_j\approx \ g\left(\prescript{(l)}{\mathrm{in}}{P}_j\right)$ using Lemma \ref{lemma_taylor_poly}.
    \EndFor
    \For{$j=1, \cdots, h_{l+1}$} 
        \State \textbf{Compute} $\prescript{(l+1)}{\mathrm{in}}{P}_j = \sum_{i=0}^{h_{l}}\prescript{(l+1)}{}{w}_{ij}\prescript{(l)}{\mathrm{out}}{P}_i$
    \EndFor
\EndFor
\If{Non-linear output at last layer}
    \For{$j=1, \cdots, h_{L+1}$} 
        \State \textbf{Compute} $\prescript{(L)}{\mathrm{out}}{P}_j\approx  g\left(\prescript{(L)}{\mathrm{in}}{P}_j\right)$ using Lemma \ref{lemma_taylor_poly}.
    \EndFor
\EndIf
\end{algorithmic}
\end{algorithm}

An important remark has to be made about the obtained polynomial orders. For $l=1,\dots,L$, the truncation order of the Taylor expansion performed at layer $l$ to obtain $\prescript{(l)}{\mathrm{out}}{P}_j$ using Lemma \ref{lemma_taylor_poly} is denoted by $q_l$.

Then the order of $\prescript{(l)}{\mathrm{out}}{P}_j$ is:
\begin{equation}
    Q_{l} = \prod_{i=1}^{l} q_i,
    \label{eq_polynomial_order}
\end{equation}
and the order of $\prescript{(l)}{\mathrm{in}}{P}_j$ is $Q_{l-1}$, 
as the Taylor expansion has not yet been applied to the input.

The final expressions of the input and output coefficients of variables $\Vec{t}$ at layer $l$ and neuron $j$, are:
\begin{align}
    \prescript{(l)}{\mathrm{in}}{\beta}_{\vec{t}} &= 
    \sum_{i=0}^{h_{l-1}}\prescript{(l)}{}{w}_{ij}   \prescript{(l-1)}{\mathrm{out}}{\beta}_{\vec{t},i} \\
    \prescript{(l)}{\mathrm{out}}{\beta}_{\vec{t}} &= 
    \sum_{n=0}^{q_l}\dfrac{\prescript{(l)}{}{g}^{(n)}(0)}{n!}
    \sum_{ \Vec{n} \in \pi(\Vec{t},Q_{l-1},n)}
    \left(\begin{array}{c} n \\ \Vec{n} \end{array}\right)\prod_{k=1}^{N_{p,Q_{l-1}}} \prescript{}{\mathrm{in}}{\beta}_{k}^{n_{k}}.
\end{align}

\section{Practical implementation}
\label{section_practical_implementation}

In this section, some practical considerations about the implementation of NN2Poly are presented and discussed. First, given that it is common practice to scale the input of a NN to the $[-1, 1]$ range, our method uses a Taylor expansion around $0$. The validity of this approximation depends on the Taylor order as well as the activation function. Therefore, some constraints need to be imposed on the NN weights in order to limit the growth of the synaptic potentials to the range in which this Taylor expansion is well-behaved. Second, the polynomial order grows too fast when using all possible coefficients obtained from successive layers. In practice, it is often observed that the higher-order terms in the polynomial model are negligible, leading to a situation where imposing a maximum order does not result in any significant loss of accuracy. Finally, a simplified approach is presented to avoid computing all the multiset partitions. A reference NN2Poly implementation is available, as an R package, at \url{https://github.com/IBiDat/nn2poly}, and all the results are available at \url{https://github.com/IBiDat/nn2poly.paper}.

\subsection{Validity of Taylor approximations} \label{section_constraints}
The practical application of NN2Poly relies heavily on the synaptic potentials at each neuron being close enough to $0$, where the performed Taylor expansion is centered and therefore remains accurate. In this context, there are several regularization techniques that are usually employed when training NNs which can be useful to overcome this problem. In particular, restricting the norm of the hidden layer weight vectors while leaving the last layer weights unrestricted is used in \cite{emschwillerNeuralNetworksPolynomial2020}. This approach to ensure that the Taylor approximation is in the desired range was already explored in \cite{moralaMathematicalFrameworkInform2021} in the case of a single hidden layer NN. In this approach, it is proposed that the NN is trained  satisfying that the vector weights in the single hidden layer ($l=1$) are constrained to have an $\ell_1$-norm equal to 1:
\begin{equation}
    \left\lvert \left\lvert \prescript{(1)}{}{\vec{w}}_{j} \right\rvert \right\rvert_1 = \sum_{i=0}^p \left\lvert \prescript{(1)}{}{w}_{ij} \right\rvert = 1
\end{equation}
for all $j =1,\dots,h_1$. This restriction is especially useful in the algorithm presented in \cite{moralaMathematicalFrameworkInform2021}, because,  combining it with the scaling of the input data to the $[-1,1]$ interval, it allows us to also constrain the activation potentials to be $ \left\lvert \prescript{(1)}{}{u}_{j} \right\rvert \leq 1$, as it is shown as follows:
\begin{multline*}
    \left\lvert \prescript{(1)}{}{u_{j}} \right\rvert = \left\lvert \sum_{i=0}^p \prescript{(1)}{}{w}_{ij} x_{i} \right\rvert \leq  \sum_{i=0}^p \left\lvert \prescript{(1)}{}{w}_{ij} x_{i} \right\rvert  \\
    \leq \sum_{i=0}^p \left\lvert \prescript{(1)}{}{w}_{ij} \right\rvert = \left\lvert \left\lvert \prescript{(1)}{}{w}_{ij} \right\rvert \right\rvert_1 = 1,
\end{multline*}
where $\left\lvert x_{i} \right\rvert \leq 1$ due to the $[-1,1]$ input data scaling.

In order to extend this condition to every hidden layer $l$, so it holds that $\left\lvert \prescript{(l)}{}{u_{j}}\right\rvert \leq 1$, an extra consideration has to be made. The input $\prescript{(l-1)}{}{y_{i}}$ to layer $l$ may no longer satisfy $\left\lvert \prescript{(l-1)}{}{y_{i}} \right\rvert \leq 1$ for $l \geq 2$. To ensure this condition, the activation function $\prescript{(l-1)}{}{g}$ can be chosen such that its output remains in $[-1,1]$, like the hyperbolic tangent $\tanh()$. This holds because $\prescript{(l-1)}{}{y_{i}} = \prescript{(l-1)}{}{g}\left(\prescript{(l-1)}{}{u_{i}}\right)$, so a $\prescript{(l-1)}{}{g}()$ that satisfies this will ensure that $\left\lvert \prescript{(l-1)}{}{y_{i}} \right\rvert \leq 1$. Furthermore, a less strict condition can be used: assuming $ \left\lvert \prescript{(l-1)}{}{u}_{i} \right\rvert \leq 1$, any activation function $\prescript{(l-1)}{}{g}$ that maps $[-1,1]$ into itself will produce the same result. Therefore, a NN trained with activation functions that satisfy this and with a constraint that forces the weights to have an $\ell_1$-norm equal to 1 will ensure that the synaptic potentials $\prescript{(l)}{}{u}_j$ will be constrained, avoiding errors in the Taylor expansion. Note also that the weights in the output layer do not need to be restricted when solving a regression problem, as no Taylor expansion is performed in the output layer.

Other regularization techniques, such as the  $\ell_2$-norm, or different bounds might be employed to solve the problem of controlling the Taylor expansion and the synaptic potentials. However, it should be noted that there is a trade-off between the explainability using NN2Poly with these constraints and the speed and effectiveness of the NN training. 

\subsection{Polynomial order and number of interactions}
\label{section_avoid_high_order}

With the general procedure developed in Section \ref{section_nn2poly_theoretical}, the polynomial order grows, as the product of the Taylor truncation order used at each layer (see Eq. \ref{eq_polynomial_order}), which generates a highly increasing number of interactions. Fig. \ref{fig_n_terms_polynomial} shows the growth of the number of polynomial terms with respect to the number of variables $p$ and the polynomial order $Q$. 
\begin{figure}[t]
	\centering
	\includegraphics[width = \linewidth]{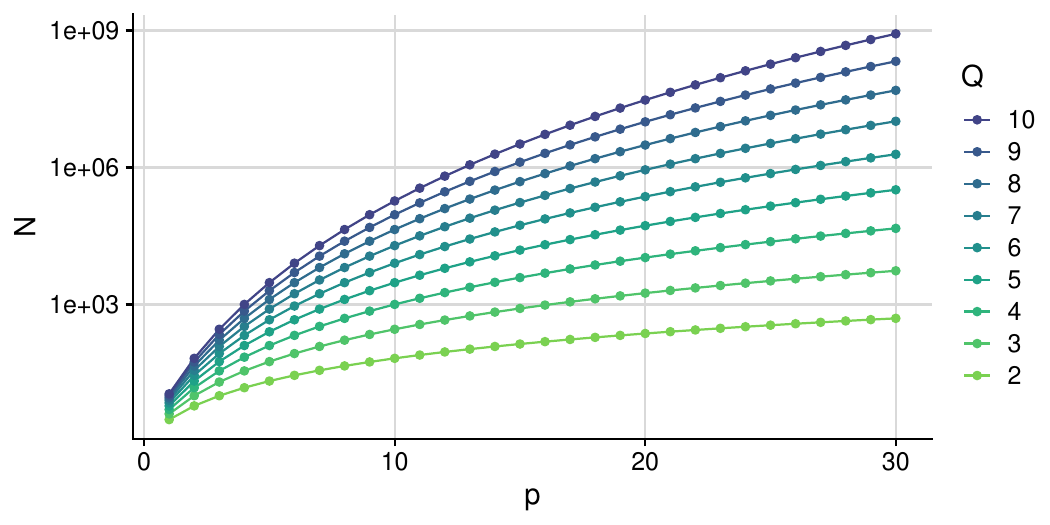}
	\caption{Number of terms, in logarithmic scale, in a polynomial, varying number of variables $p$ and order of the polynomial $Q$, with all interactions being considered.}
	\label{fig_n_terms_polynomial}
\end{figure}
This raises two issues. First, in terms of computational complexity, as the problem quickly becomes infeasible for deep learning architectures. Second, the interpretability obtained from those coefficients becomes less useful when there are too many of them. 

However, the relationships that the NN might learn from the data may not be so complex for many problems. In practice, the obtained higher order coefficients tend to be close to zero and not relevant in the polynomial prediction (see Section \ref{section_simulation_study}). This can be exploited to impose an overall truncation order $Q_{\mathrm{max}}$ at every layer, which hugely simplifies the computations. As an example, let us consider a case where the input data has quadratic relations and the NN trained with it has 4 hidden layers. Using NN2Poly with Taylor truncation order $q_i=3$, at all the hidden layers, would yield a final polynomial of order $Q=3^4=81$, which is obviously excessive for quadratic data. Therefore, it is expected that the weights learned by the NN will be such that the final function represents those quadratic relations, and thus the higher order coefficients will be close to zero. Then, setting $Q_{\mathrm{max}=2}$ or $Q_{\mathrm{max}=3}$ will suffice, avoiding the computation of higher order terms. Then, the new polynomial order obtained at each layer $l$ will be
\begin{equation}
    Q_{l}^{*} = \min\left(\prod_{i=1}^{l} q_i,Q_{\mathrm{max}}\right).
    \label{eq_forced_max_Q}
\end{equation}
This new condition is included in Algorithm \ref{alg_NN2Poly_practical} to make a more practical and efficient version of Algorithm \ref{alg_NN2Poly}.

\subsection{Multiset partitions generation and selection}\label{section_multiset_computation}

One of the main computational efforts of the presented method is the computation of all the allowed partitions of a given multiset $M$ associated to $\Vec{t}$ to obtain $\pi(\Vec{t},Q,n)$ from Lemma \ref{lemma_taylor_poly}. In general, finding all the partitions of a multiset can be done with Algorithm \ref{alg_knuth}, presented in Appendix \ref{appendix_multiset} as developed in \cite{knuthArtComputerProgramming2005}. However, as these partitions have to be computed for each term in the polynomial, this can become quite computationally expensive when the polynomial grows both in size $p$ and order $Q$, as was shown in Fig. \ref{fig_n_terms_polynomial}. Furthermore, the  partitions of higher order terms become significantly harder to compute. The top panels in Fig. \ref{fig_partitions_time_and_size} show the computational time and memory size required to compute and store the multiset partitions, associated to each term in a polynomial, and with several values of order $Q$ and size $p$.
\begin{figure}[t]
	\centering
	\includegraphics[width = \linewidth]{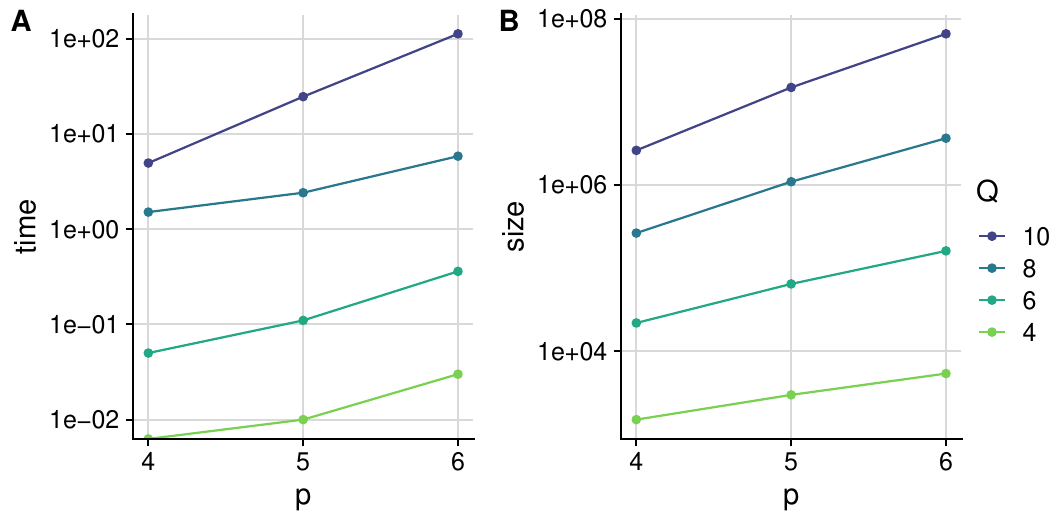}
	\includegraphics[width = \linewidth]{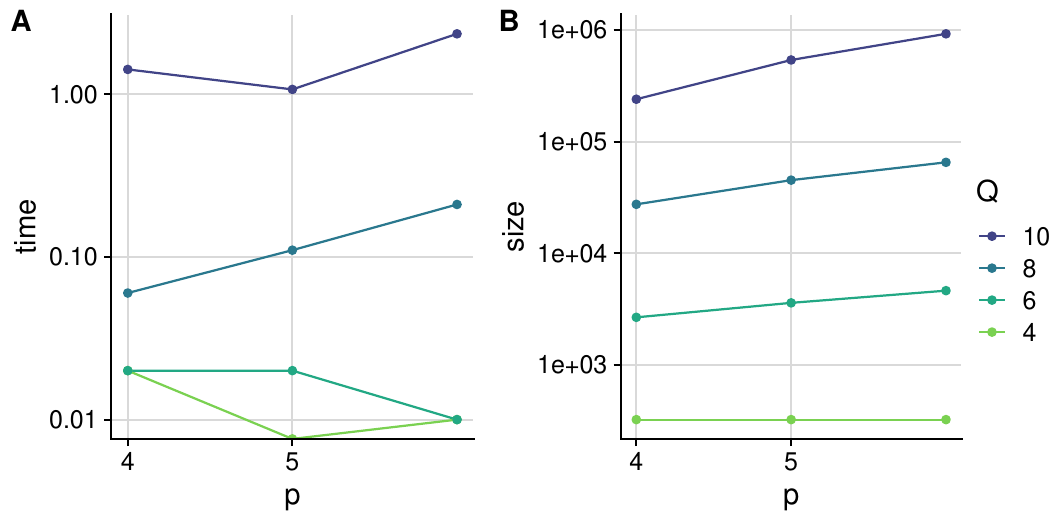}
	\caption{Computational time (A) and size (B), in logarithmic scale, required to compute all the possible multiset partitions 1) for all terms in a polynomial of order $Q$ with $p$ variables (top panels), and 2) only for the representative terms $M_0$ (defined in Section \ref{section_multiset_computation}) in a polynomial of order $Q$ with $p$ variables (bottom panels).}
	\label{fig_partitions_time_and_size}
\end{figure}

A useful idea to reduce this computational burden is to identify partitions that are equivalent, which happens when two multisets 
have the same number of unique elements and same multiplicities, i.e., there exists a bijection between the elements of both multisets. Consider the multiset $M_1=\{1,1,2,3\}$, for which its partitions have been computed, and another multiset $M_2=\{2,3,3,5\}$. Then, performing the changes $3\rightarrow 1$, $2\rightarrow 2$ and $5\rightarrow 3$, $M_2$ can be transformed into $M_1$, and all the $M_1$ partitions can be transformed into the $M_2$ partitions. Using this idea, the number of multisets for which we need to explicitly compute their partitions can be reduced significantly.

To determine a consistent criterion to choose the multiset that will be used to compute the partitions of all equivalent multisets, consider an arbitrary multiset $M$ given by the combination of variables $\Vec{t}=(t_1,t_2,\dots,t_p)$ per Eq. \ref{eq_multiset_M_from_t}. Then, consider $p_0$ the number of components $t_i > 0$ for all $i=1,\dots,p$, and order those components greater than zero in descending order in a new vector  $\Vec{t}_{0}=(t_{(1)},t_{(2)},\dots,t_{(p_0)})$, where $t_{(1)} \geq t_{(2)} \geq \dots \geq t_{(p_0)}$ and each $t_{(j)}$ takes the value of some $t_i>0$ starting from $t_{(1)}=\max_{i=1}^p (t_i)$. Then, the multiset $M_0$ given by $\Vec{t}_{0}$, defined by
\begin{equation}
    M_0 = \left\{\underbrace{1,\dots,1}_{t_{(1)}}, \underbrace{2,\dots,2}_{t_{(2)}}, \dots, \underbrace{p,\dots,p}_{t_{(p_0)}}\right\}
\end{equation}
is equivalent to $M$ given by $\Vec{t}$. To provide a clarifying example, consider $M=\{2,3,3,5\}$, which is equivalent to $\Vec{t}=(0,1,2,0,1)$ if $p=5$. Then, the equivalent multiset is given by $\Vec{t}_0=(2,1,1)$ and $M_0=\{1,1,2,3\}$, where $t_{(1)}=2=t_3$, so $3$ is changed to $1$, $t_{(2)}=1=t_2$, so $2$ is changed to $2$ and $t_{(3)}=1=t_5$, so $5$ is changed to $3$. Note that this would yield the multiset $\{2,1,1,3\}$, but, as order does not affect multisets, it is the same multiset as $M_0=\{1,1,2,3\}$. Computing only the partitions of the multisets $M_0$ that represent each possible class of equivalence between multisets significantly decreases the computation times and storage sizes as can be seen in the bottom panels in Fig.~\ref{fig_partitions_time_and_size}. This efficiency improvement is also included in Algorithm~\ref{alg_NN2Poly_practical}.

\begin{algorithm}[t]
\small
\caption{NN2Poly: Practical implementation}\label{alg_NN2Poly_practical}
\begin{algorithmic}[1]
\Require Weight matrices $\prescript{(l)}{}{W}$, activation functions $\prescript{(l)}{}{g}$,  Taylor truncation order at each layer $q_l$, maximum order $Q_{\mathrm{max}}$.
\State Compute all multiset partitions for every equivalent $M_0$ needed for a polynomial of order $Q_{\mathrm{max}}$ in $p$ variables.

\State Set $\prescript{(1)}{\mathrm{in}}{P}_j = \prescript{(1)}{}{u}_j$ for all $j=1, \dots, h_1$.
\For{$l=1, \cdots, L-1$}
    \For{$j=1, \cdots, h_l$} 
        \State \textbf{Compute} the coefficients, only up to order $Q_{l}^{*}$, from  $\prescript{(l)}{\mathrm{out}}{P}_j\approx \ g\left(\prescript{(l)}{\mathrm{in}}{P}_j\right)$ using Lemma \ref{lemma_taylor_poly}.
    \EndFor
    \For{$j=1, \cdots, h_{l+1}$} 
        \State \textbf{Compute} the coefficients, only up to order $Q_{l}^{*}$, from $\prescript{(l+1)}{\mathrm{in}}{P}_j = \sum_{i=0}^{h_{l}}\prescript{(l+1)}{}{w}_{ij}\prescript{(l)}{\mathrm{out}}{P}_i$
    \EndFor
\EndFor
\If{Non-linear output at last layer}
    \For{$j=1, \cdots, h_{L+1}$} 
        \State \textbf{Compute} the coefficients, only up to order $Q_{l}^{*}$, from  $\prescript{(L)}{\mathrm{out}}{P}_j\approx  g\left(\prescript{(L)}{\mathrm{in}}{P}_j\right)$ using Lemma \ref{lemma_taylor_poly}.
    \EndFor
\EndIf
\end{algorithmic}
\end{algorithm}

Finally, Lemma \ref{lemma_taylor_poly} defines two conditions that would discard some of the computed partitions. These conditions are currently checked after generating all the partitions using Algorithm~\ref{alg_knuth}, but in the future could be included in the partition generation algorithm to avoid unnecessary computations. However, when including a forced maximum order for the polynomial, $Q_{\mathrm{max}}$ from Eq.~\ref{eq_forced_max_Q}, the number of discarded partitions becomes really low. As soon as $Q_{\mathrm{max}}$ is achieved at some layer, Condition 2 is no longer excluding any partition; and as long as $q_l$ is set as greater than $Q_{\mathrm{max}}$, the same happens with Condition 1.

Algorithm \ref{alg_NN2Poly_practical} implements NN2Poly including the efficiency improvements from Section \ref{section_avoid_high_order} and Section \ref{section_multiset_computation}.

\section{Simulation study and discussion}\label{section_simulation_study}

\begin{figure}[t]
	\centering
	\includegraphics[width = \linewidth]{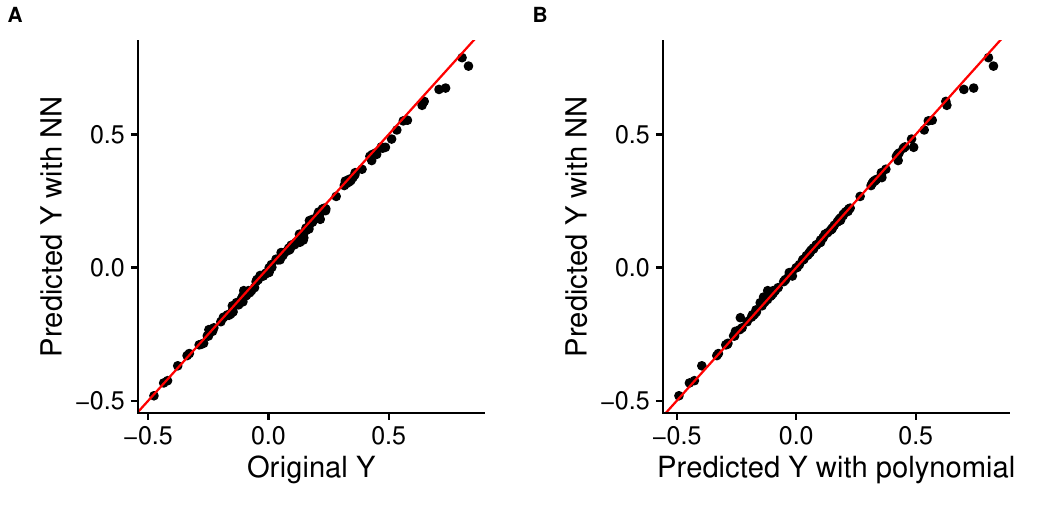}
	\includegraphics[width = \linewidth]{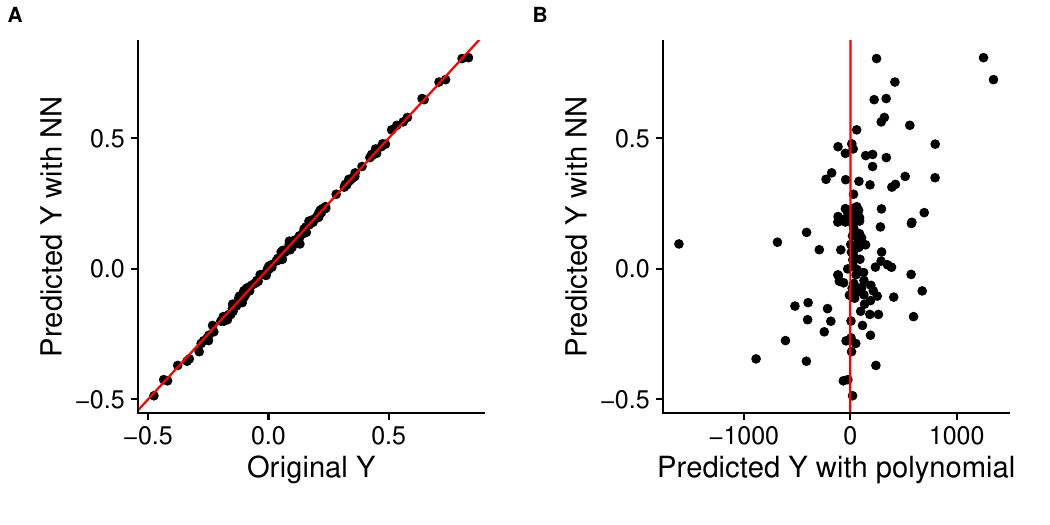}
	\caption{Performance example of NN2Poly used on two NNs to obtain a polynomial with the same predictions. The NNs are trained on polynomial data and built with 3 hidden layers, 100 neurons per hidden layer and hyperbolic tangent activation functions, imposing weight constraints (top panels) and no constraints (bottom panels). (A) represent the NN predictions vs. the original response in the data. (B) represent the NN predictions vs. the obtained polynomial predictions. }
	\label{fig_performance}
\end{figure}

In this section, some simulation results will be presented, showcasing the effectiveness of NN2Poly, especially on those networks that satisfy the $\ell_1$-norm equal to 1 constraint, presented in Section \ref{section_practical_implementation}. Notably, if any MLP is randomly initialized under such constraints, the resulting polynomial inferred by NN2Poly perfectly matches the NN predictions. Therefore, the challenge here is to demonstrate that the NN can successfully learn patterns from data and that the proposed method still can be used to understand them. To this end, the final section shows two real use cases with tabular data based on the Boston Housing data set (regression problem) \cite{harrisonHedonicHousingPrices1978} and the Forest Covertype data set (classification problem) \cite{blackardComparativeAccuraciesArtificial1999}. Tabular data has been identified still as an important challenge for neural networks \cite{borisovDeepNeuralNetworks2022}, and NN2Poly can be a useful tool in this line of research. The chosen real data sets (Boston Housing and Covertype) have been chosen following the recommendations for tabular data benchmarks and neural networks given in \cite{borisovDeepNeuralNetworks2022}. Other types of data, such as images (where each pixel would represent a variable), require future work when adapting the computational implementation, because, as explained in section \ref{section_avoid_high_order}, the number of variables considered is critical when scaling the method.

\subsection{Performance example}\label{section_performance}

\begin{figure*}[t]
	\centering
	\includegraphics[width = \linewidth]{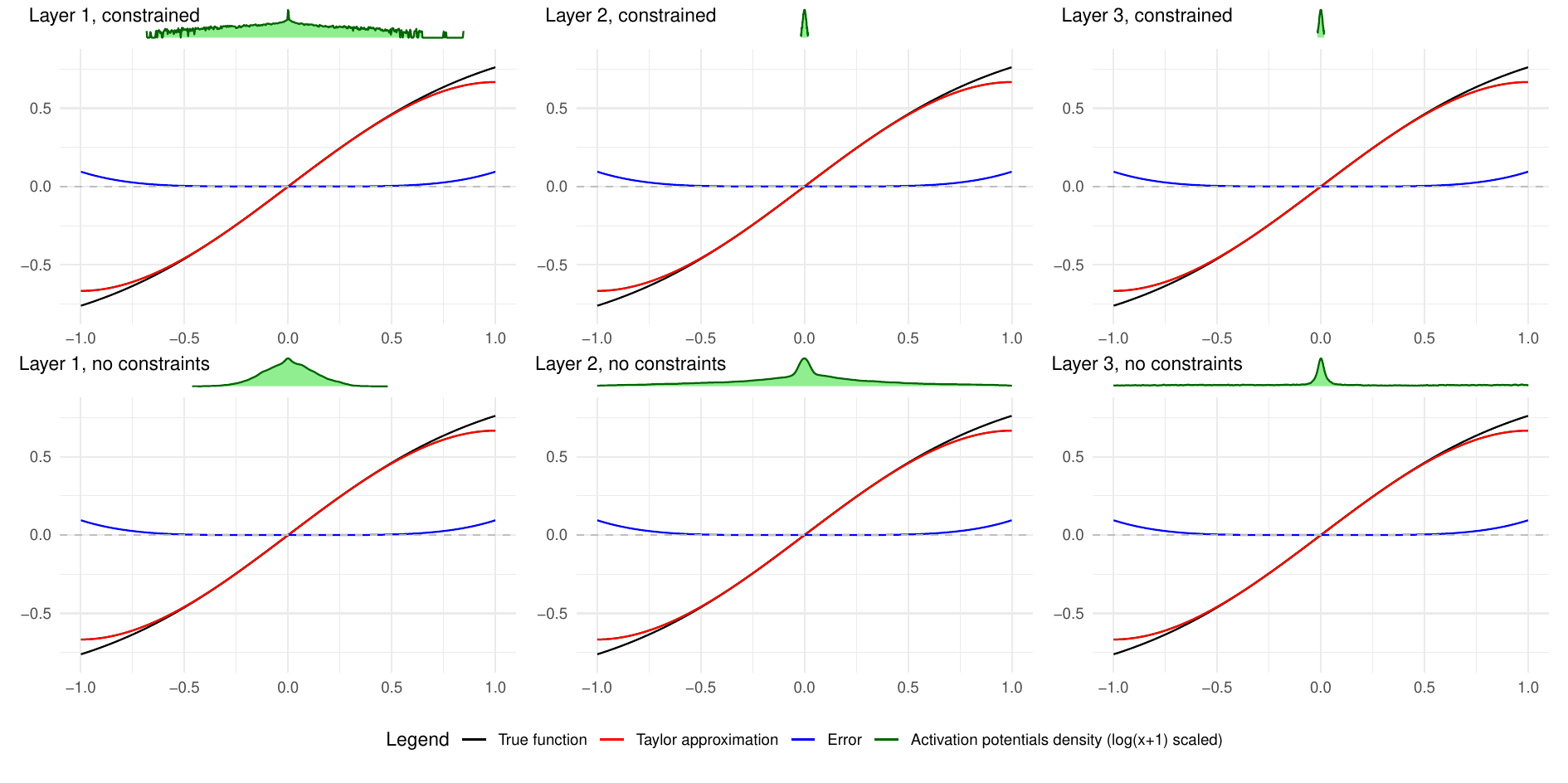}
	\caption{Taylor expansion analysis by layers for the two NNs from Fig. \ref{fig_performance}, one with weight constraints (top) and the other with no constraints (bottom). The blue line represents the absolute error between the Taylor approximation (red) and the actual activation function (black). The green line on top of each plot represents the density distribution of the synaptic potentials obtained from the test data (transformed using $log(x+1)$ to better visualize the results). The constrained NN (top) does a better job of keeping the potentials inside the acceptable region.}
	\label{fig_taylor_layers}
\end{figure*}

\begin{figure}[t]
	\centering
	\includegraphics[width = \linewidth]{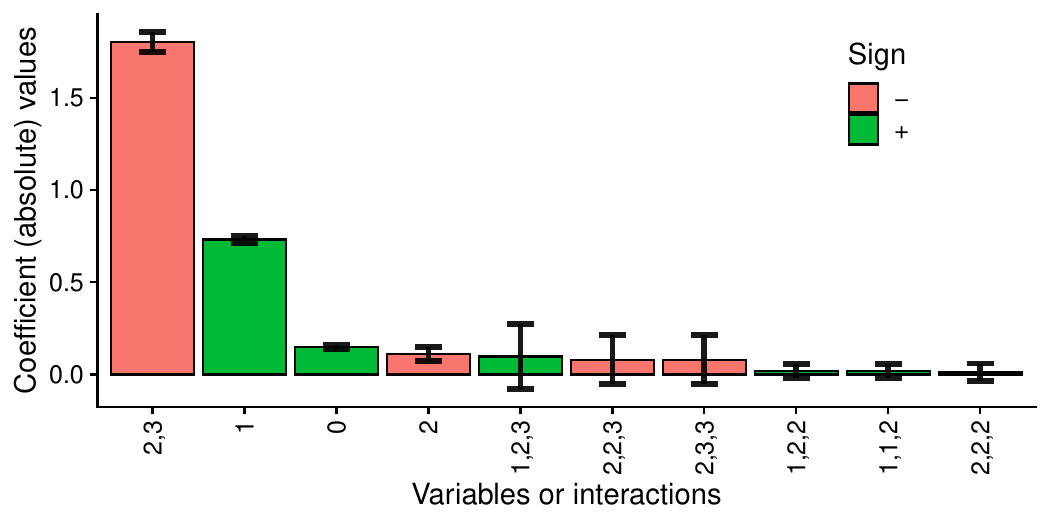}
	\caption{Average polynomial coefficients obtained after applying NN2poly on 10 NNs trained with the same structure as in Fig \ref{fig_performance}. The 10 highest average coefficients, scaled and in absolute value are shown. The two principal coefficients are "2,3" and "1", corresponding to the terms in the original polynomial that generated the data. Variables or interactions are represented with numbers in the horizontal axis, with red being negative coefficients and green positive ones.}
	\label{fig_coeffs_example1}
\end{figure}

Before proceeding with the complete simulation study, two performance examples are presented in detail to motivate and understand the results that can be expected from NN2Poly in different situations.
For illustration purposes, a small polynomial of order 2, in 3 variables, is used to generate data: $4X_1 - 3 X_2X_3$. Each variable $X_i$ is generated from a $\mathcal{N}(0,1)$. The response $Y$ is computed using the polynomial an adding a small error from a $\mathcal{N}(0,0.1)$. The process is repeated 500 times. The data set is randomly divided in a $75-25\%$ split for training and test, and subsequently scaled to the $[-1,1]$ interval.

With this data, two NN have been trained, with the same hyperparameters, differing only in the fact that one of the NN has the weights constrained to have $\ell_1$-norm equal to 1, as explained in Section \ref{section_constraints}. Both NNs have 3 hidden layers with 100 hidden neurons per layer, include the hyperbolic tangent as activation function, and they are trained with the ADAM optimizer \cite{kingmaAdamMethodStochastic2017}. NN2poly is then employed on both NNs, using $Q_{\mathrm{max}}=3$.

Fig. \ref{fig_performance} shows the performance for the NN trained with $\ell_1$-norm weight constraints (top) and the NN trained with no weight constraints (bottom). The left panels (A) compare the NN predictions to the original response $Y$, and the right panels (B) compare the NN predictions to the polynomial predictions obtained using NN2Poly. Note here that the NN performance is not directly relevant to the NN2Poly performance, as the method will represent the NN behavior without taking into consideration how well the NN solves the given problem. Regarding this, it should be noted how the polynomial predictions closely match the NN in the constrained case (top right) opposed to the unconstrained situation (bottom right). However, the NN learns more slowly when trained under these constraints. This is a clear example of the existing trade-off discussed in Section~\ref{section_constraints}.

Fig. \ref{fig_taylor_layers} represents the behavior of the Taylor approximation at each layer in the two previous NNs (constrained and unconstrained). The black line represents the actual activation function at that layer, the red line, its Taylor approximation, and the blue line, the error computed as the absolute value between them. Then, the green line represents the density distribution of the synaptic potentials, obtained from the test data. The narrower this density is centered around zero, the closer the Taylor expansion is to the real activation function in most of the data points. In these examples, both have their synaptic potentials centered around zero, but the restricted case quickly reduces its range in the deeper layers while the unrestricted case has synaptic potentials that are wider and increase their range, which is a behavior that impairs NN2Poly's representation accuracy.

Fig. \ref{fig_coeffs_example1} corresponds to training ten NNs on the same data and same conditions as before, in the constrained case. The polynomial coefficients are averaged over the ten realizations and the obtained means are represented in absolute value in descending order and including their standard deviations as error bars. Only the 10 highest coefficients are represented. It is clear that the obtained polynomial is an accurate representation of the original data, as the two most important coefficients are "2,3" (the interaction between variables 2 and 3) and "1" (variable 1 alone), corresponding to the two terms in the original polynomial that generated the data. The rest of the coefficients are small,  and appear due to the error term added when generating the data and because of the approximation nature of the polynomial. Note that the coefficients are scaled because training data falls within the $[-1,1]$ interval, while the original polynomial coefficients were defined before scaling. Therefore, they retain equivalent magnitudes but not the exact same values.

\subsection{Simulation study}

\begin{figure*}[t]
	\centering
	\includegraphics[width = \linewidth]{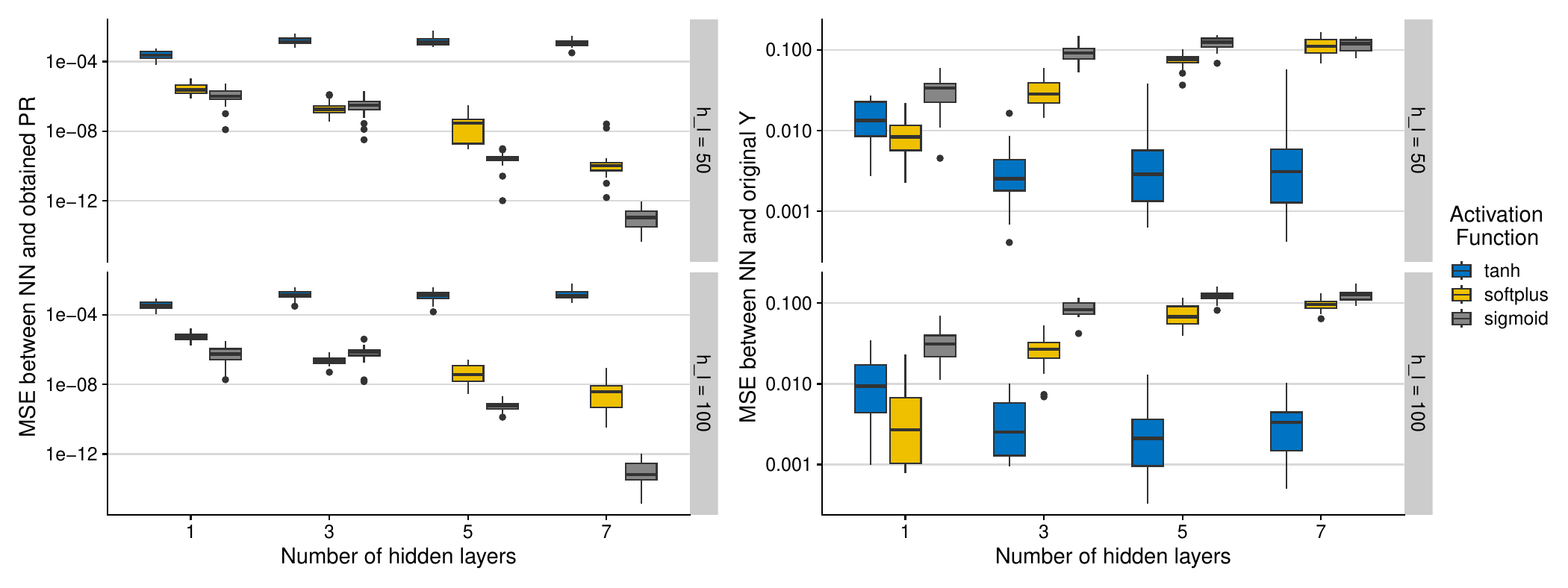}
	\caption{MSE between NN predictions and 1) the obtained polynomial predictions with NN2Poly (left), and 2) the original response $Y$ (right) for several NN configurations: 1, 3, 5 and 7 layers, 50 and 100 units per hidden layer, and hyperbolic tangent (blue), softplus (yellow) and sigmoid (grey) as activation functions.}
	\label{fig_MSE}
\end{figure*}

In order to address the effectiveness of NN2Poly over several examples, the method is tested using different MLP configurations in a regression setting, trained over simulated data. In these examples, the data is generated from polynomials of order 2 in 5 variables ($p=5$), limiting the number of possible interactions of order 2 to be 5 out of all possible, chosen at random. The coefficients of the selected terms in the polynomial are taken at random from $\{-2,1,2,1\}$. The predictor variables $X_i$ are generated from a uniform distribution in $[-10,10]$, taking 500 samples. Finally, the output $Y$ is computed using the selected polynomial coefficients and predictor variables and adding an error $\epsilon \sim \mathcal{N}(0,1)$. Then, the data set is randomly divided in a $75-25\%$ split for training and test, and subsequently scaled to the $[-1,1]$ interval. The NNs are created with 1, 3, 5 and 7 hidden layers, with 50 and 100 neurons per hidden layer, and different activation functions are tested (hyperbolic tangent, softplus and sigmoid) with the $\ell_1$-norm weight constraint.

The NN2Poly method is applied after training, using Taylor expansion of order $8$ at each layer and $Q_{\mathrm{max}}=3$, obtaining the coefficients of a polynomial representation. This process is repeated 20 times for each NN configuration, and each time, the Maximum Squared Error (MSE) between the obtained polynomial predictions and the NN predictions is computed and represented in Fig. \ref{fig_MSE} (left). The MSE between the NN and the original response is also represented in Fig. \ref{fig_MSE} (right). It can be observed that the representation error achieved with NN2Poly is quite low, then slightly higher and increasing with the number of layers in the case of the hyperbolic tangent. On the other hand, the softplus and sigmoid examples seem to get lower errors and improve with the number of layers. However, it should be noted that NNs using both the softplus and the sigmoid seem to have problems when actually learning the original response $Y$, as can be seen in Fig. \ref{fig_MSE} (right), where the NN performance does only achieve satisfactory results in the case of the hyperbolic tangent. 

In terms of the variability, Fig. \ref{fig_MSE} (left) shows that the NN2Poly MSE has higher variability when the number of layers is increased for the softplus and sigmoid cases, while the hyperbolic tangent does not show this behavior, even reducing its variability with more layers. This difference can be due to the bounds and the symmetry around 0 of the hyperbolic tangent's output, while the sigmoid and softplus response characteristics may affect subsequent layers when using a Taylor expansion cantered around zero.

\subsection{Real application: Boston Housing data set (Regression)}\label{section_boston}
In this section, a real use case of NN2Poly is showcased using the Boston Housing data set \cite{harrisonHedonicHousingPrices1978}. It presents a regression problem aimed at predicting the median value of homes in a Boston suburb in the 1970s. The data set contains 506 rows, split in 404 and 102, train and test respectively. See Table~\ref{tab:boston} for further details.

With this data, 10 NNs are trained using the same methodology described in Section \ref{section_performance}, with 3 hidden layers, 100 neurons per layer and including the hyperbolic tangent as activation function. The weights are updated using ADAM \cite{kingmaAdamMethodStochastic2017}, imposing the $\ell_1$-norm weight constraint. Then, NN2Poly is applied with $Q_{\mathrm{max}}=3$ to obtain a polynomial representation. The NN's performance for one of these realizations is depicted in Fig.~\ref{fig_Boston_performance}A (all of them achieve similar results). On the other hand, Fig.~\ref{fig_Boston_performance}B shows that the polynomial representation obtained by NN2Poly closely matches the NN predictions.

\begin{figure}[t]
	\centering
	\includegraphics[width = \linewidth]{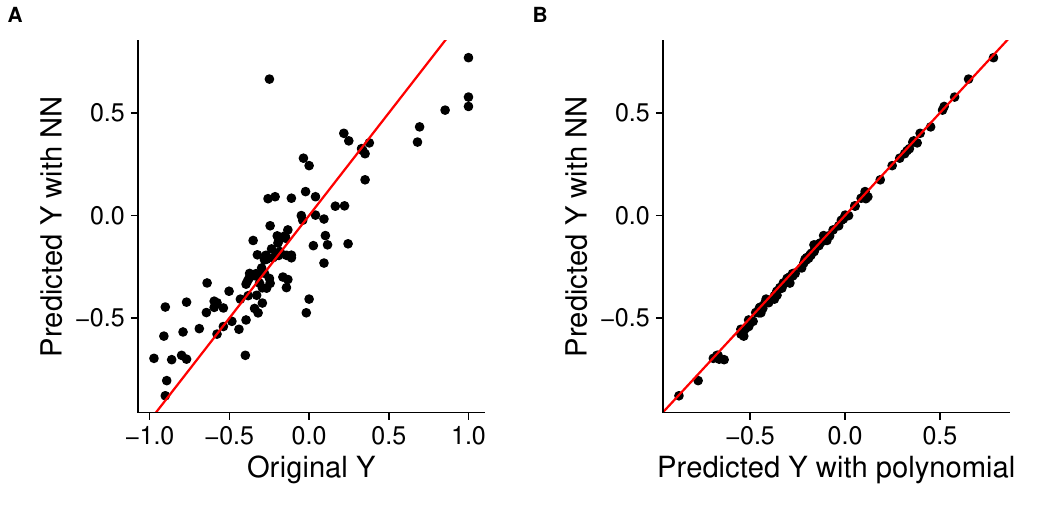}
	\caption{Example of the predictions obtained by one of the NNs trained with the Boston data set. (A) represents the NN predictions against the original response. (B) represents the NN predictions versus the obtained polynomial predictions using NN2Poly.}
	\label{fig_Boston_performance}
\end{figure}

The mean and the standard deviation of the coefficients of this polynomial are obtained over the 10 realizations. The most important 20 coefficients are shown in Fig.~\ref{fig_Boston_barplot}, where 0 represents the intercept, 1 represents the first variable, and so on (see Table~\ref{tab:boston} for their meanings). This allows us to observe which variables (or interactions) are more important, interpret their meaning, and analyze their corresponding variation with different trainings of the NN. Note that the obtained results are consistent with the expected behavior both in importance (absolute value) and direction (coefficient sign). For example, the four most important variables are 13 (percentage lower status of the population, appears in the polynomial with negative sign), 6 (the average number of rooms per dwelling, appears with positive sign), 1 (per capita crime rate, appears with negative sign) and 8 (weighted distances to five Boston employment centers, appears with negative sign). On the other hand, variables like 7 (the proportion of owner-occupied units built before 1940) do not appear in the selected coefficients, meaning that it affects the output less, which is coherent as this variable does not seem as relevant as the previous ones.

\begin{figure}[t]
	\centering
	\includegraphics[width = \linewidth]{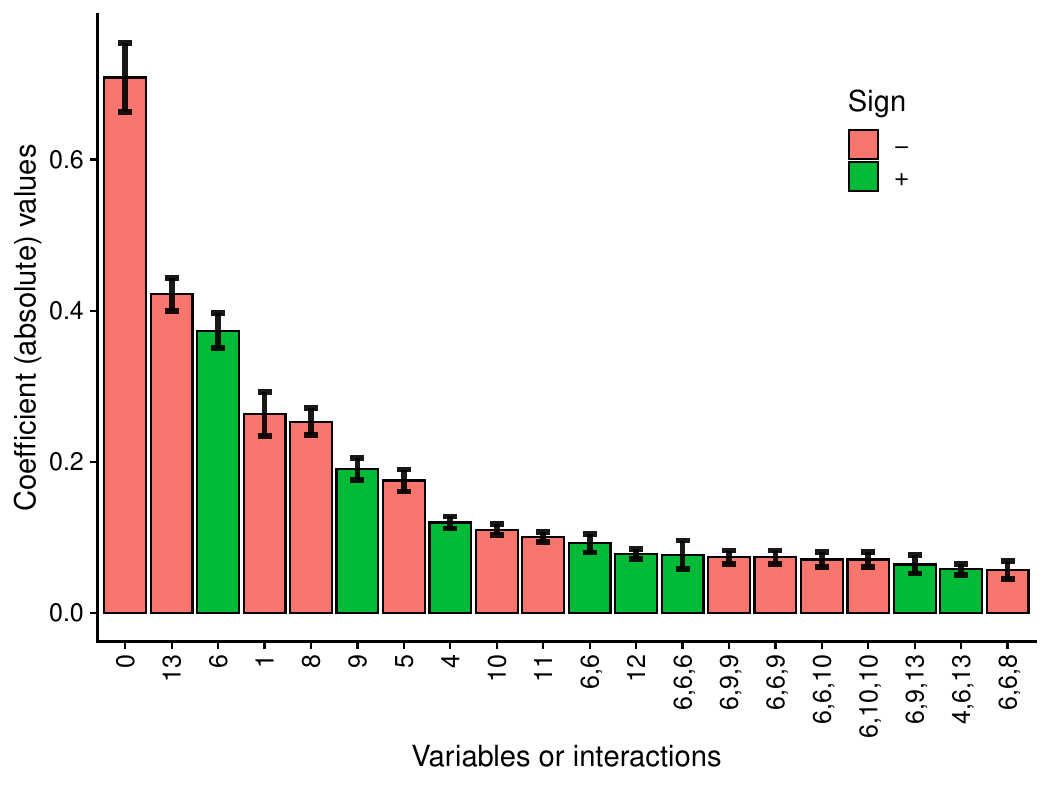}
	\caption{Average polynomial coefficients obtained after applying NN2Poly on the 10 NNs trained with the Boston data set. The 20 highest average coefficients, scaled and in absolute value are shown. Variables or interactions are represented with numbers in the horizontal axis, with red being negative coefficients and green positive ones.}
	\label{fig_Boston_barplot}
\end{figure}

\subsection{Real application: Covertype data set (Classification)}\label{section_covertype}

In this section, a different real data set is utilized to train a NN, and subsequently, NN2Poly is employed to derive the polynomial representation in a classification setting, showcasing how distinct polynomials are obtained for each class. The Covertype data set \cite{blackardComparativeAccuraciesArtificial1999} contains cartographic data  from four wilderness areas in the Roosevelt National Forest, Colorado, United States. The observations correspond to 30 squared meters cells, were one out of 7 forest types is assigned. Each of these observations contains 54 predictor variables, corresponding to both quantitative and qualitative variables (in binary form). The data set contains $581012$ observations, where $80\%$ is used for training and the rest for testing, and all variables have been scaled to the $[-1,1]$ interval.

With this data set, a NN is trained with 3 hidden layers and 100 neurons each, using the hyperbolic tangent as the activation function. Given that there are 7 forest types as targets, the final layer has 7 outputs. The NN is trained using a sparse categorical crossentropy loss. This final layer does not present an activation function when training, as explained in \ref{section_nn2poly_theoretical}. Then, the probabilities for each observation and class are obtained with the final trained model adding a softmax activation layer. The actual class predictions will be then obtained taking the maximum of those probabilities. This yields an accuracy of $0.7169$, which is similar to the accuracy for a simple feed forward neural network of $0.7058$ reported in \cite{blackardComparativeAccuraciesArtificial1999}. Note that this is the original NN accuracy, which is not relevant when using NN2Poly to represent it as polynomials. The confusion matrix for the test data can be seen in Table \ref{tab:confusionNN}. It is worth noting that some classes can be particularly challenging to distinguish, such as class 5. There are several factors that can contribute to the difficulty in correctly predicting this class, including the presence of a class imbalance, a poor feature representation, and the inherent complexity of the class itself.

\begin{table}[t]
\footnotesize
\centering
\caption{Confusion Matrix between NN predictions (NN, rows) and test data response (Y, columns) for each class (1 to 7)}
\begin{tabular}{cccccccc}
\hline
\multicolumn{1}{c|}{\textbf{NN/Y}} & \multicolumn{1}{c}{\textbf{1}} & \multicolumn{1}{c}{\textbf{2}} & \multicolumn{1}{c}{\textbf{3}} & \multicolumn{1}{c}{\textbf{4}} & \multicolumn{1}{c}{\textbf{5}} & \multicolumn{1}{c}{\textbf{6}} & \multicolumn{1}{c}{\textbf{7}} \\ \hline
\multicolumn{1}{c|}{\textbf{1}} & 30926 & 11290 & 0 & 0 & 20 & 0 & 2541 \\
\multicolumn{1}{c|}{\textbf{2}} & 11229 & 44340 & 817 & 0 & 1872 & 870 & 34 \\
\multicolumn{1}{c|}{\textbf{3}} & 9 & 528 & 5278 & 314 & 45 & 1565 & 0 \\
\multicolumn{1}{c|}{\textbf{4}} & 0 & 0 & 77 & 189 & 0 & 4 & 0 \\
\multicolumn{1}{c|}{\textbf{5}} & 0 & 0 & 0 & 0 & 0 & 0 & 0 \\
\multicolumn{1}{c|}{\textbf{6}} & 16 & 382 & 975 & 59 & 23 & 1096 & 0 \\
\multicolumn{1}{c|}{\textbf{7}} & 226 & 0 & 0 & 0 & 0 & 0 & 1477 \\
\hline
\end{tabular}
\label{tab:confusionNN}
\end{table}

NN2Poly is used on the trained model (linear output, no softmax activation layer included at this point), and 7 different polynomials are obtained at each output neuron. The maximum order obtained for the polynomial, $Q_{\mathrm{max}}$, was tested with values 1, 2 and 3. The straight forward application of the obtained polynomials can be measured with the MSE between the original NN predictions (linear prediction, before the softmax activation function) and the linear predictions of both polynomials. This results are in Table \ref{tab:MSE}, where it can be seen that $Q_{\mathrm{max}}=1$ and $Q_{\mathrm{max}}=2$ have similar results, with a slightly lower MSE for $Q_{\mathrm{max}}=2$. However, $Q_{\mathrm{max}}=3$ presents a significant increase in MSE. This can be explained due to the fact that increasing the Taylor order can increase the range in which the approximation is correct, but at the same time, the error grows at a faster rate outside of that range\cite{moralaMathematicalFrameworkInform2021}. The polynomial predictions are subjected to the same process as the neural network predictions. First, the predictions are converted to probabilities using the softmax function, and then, the class with the highest probability is assigned as the predicted class. The obtained accuracy is $0.99342$, for the polynomial with $Q_{\mathrm{max}}=2$, $0.99339$ for the polynomial with $Q_{\mathrm{max}}=2$, and $0.95311$ with $Q_{\mathrm{max}}=3$, due to the increase in MSE explained previously. The confusion matrix for $Q_{\mathrm{max}}=1$ can be see in Table \ref{tab:confusion1}. This result shows that NN2Poly can accurately obtain polynomial representations that predict in almost the same way as the original neural network. In this case, it also showcases that even in real scenarios with tabular data, a first or second order polynomial is enough to capture the NN's behavior. The rest of the analysis will be carried out using the $Q_{\mathrm{max}}=1$ polynomial. Note also that the number of terms in these polynomials grows as follows: 55 terms for $Q_{\mathrm{max}}=1$, 1540 terms for $Q_{\mathrm{max}}=2$ and 29260 terms for $Q_{\mathrm{max}}=3$, Therefore, if the lower order polynomial is accurate enough, it should always be preferred.

\begin{table}
\footnotesize
\caption{MSE between NN and polynomial predictions\\for each forest type (previous to softmax application)}
\centering
\begin{tabular}[t]{cccc}
\hline
\makecell{Forest \\ Type} & \makecell{MSE \\ $Q_{\mathrm{max}}=1$} & \makecell{MSE \\ $Q_{\mathrm{max}}=2$} & \makecell{MSE \\ $Q_{\mathrm{max}}=3$}\\
\hline
1 & 0.004521 & 0.004354 & 0.050123\\
2 & 0.000040 & 0.000035 & 0.000139\\
3 & 0.021703 & 0.020842 & 0.236947\\
4 & 0.042978 & 0.040572 & 0.418598\\
5 & 0.001211 & 0.001142 & 0.011137\\
6 & 0.013885 & 0.013265 & 0.145349\\
7 & 0.018609 & 0.018115 & 0.224889\\
\hline
\end{tabular}
\label{tab:MSE}
\end{table}

\begin{table}[t]
\footnotesize
\centering
\caption{Confusion Matrix between NN (rows) and polynomial (columns) predictions (using $Q_{\mathrm{max}}=1$)}
\begin{tabular}{ccccccc}
\hline
\multicolumn{1}{c|}{\textbf{NN/Poly}} & \multicolumn{1}{c}{\textbf{1}} & \multicolumn{1}{c}{\textbf{2}} & \multicolumn{1}{c}{\textbf{3}} & \multicolumn{1}{c}{\textbf{4}} & \multicolumn{1}{c}{\textbf{6}} & \multicolumn{1}{c}{\textbf{7}} \\ \hline
\multicolumn{1}{c|}{\textbf{1}} & 44479 & 0 & 0 & 0 & 0 & 298 \\
\multicolumn{1}{c|}{\textbf{2}} & 232 & 58861 & 46 & 0 & 23 & 0 \\
\multicolumn{1}{c|}{\textbf{3}} & 0 & 0 & 7685 & 54 & 0 & 0 \\
\multicolumn{1}{c|}{\textbf{4}} & 0 & 0 & 0 & 270 & 0 & 0 \\
\multicolumn{1}{c|}{\textbf{6}} & 0 & 0 & 107 & 4 & 2440 & 0 \\
\multicolumn{1}{c|}{\textbf{7}} & 0 & 0 & 0 & 0 & 0 & 1703 \\
\hline
\end{tabular}
\label{tab:confusion1}
\end{table}

\begin{figure}[t]
	\centering
	\includegraphics[width = \linewidth]{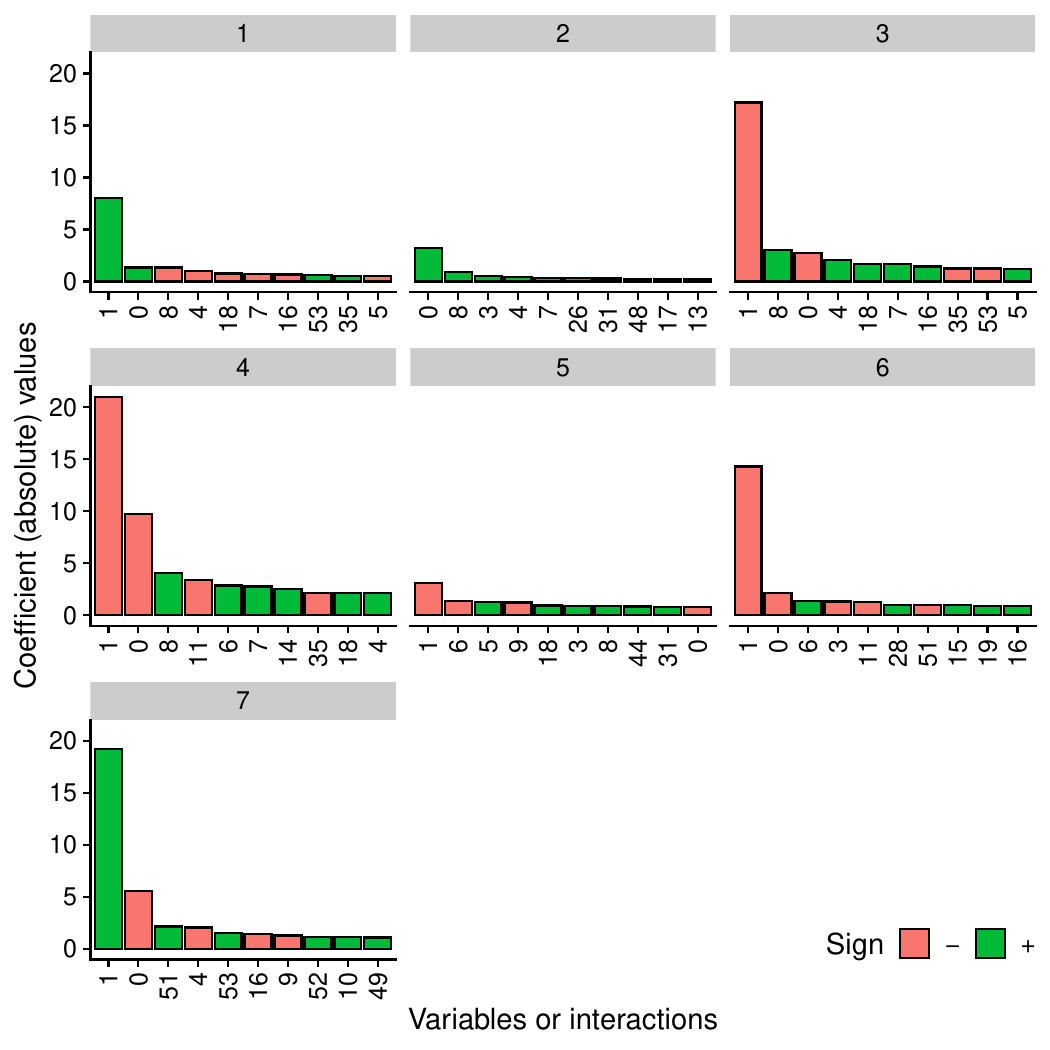}
	\caption{Polynomial coefficients obtained for each forest type (1-7) after applying NN2Poly on a NN trained on the Covertype data set. For each forest type, a different polynomial is obtained and the ten highest coefficients in absolute value are shown. Variables or interactions are represented with numbers in the horizontal axis, with red being negative coefficients and green positive ones.}
	\label{fig_covertype_coeffs}
\end{figure}

The obtained polynomials also increment the model interpretability, however in this case the coefficients are not linearly related with the actual prediction, as happened in the regression case. Rather, these coefficients effect on the class probabilities can be interpreted through the softmax function in an analogous manner as the classical interpretation of the coefficients in a logistic regression.  Figure \ref{fig_covertype_coeffs} contains, for visualization purposes, the 10 highest coefficients in absolute value of the polynomial associated with each class.  Note that this results are quite consistent with the forest type. The following relations exemplify this: Forest type 1 corresponds with the \textit{Spruce/Fir} forest type, which is characteristic of high elevation mountains, and its highest coefficient (positive value) is associated with variable 1, that is \textit{Elevation}. On the other hand, note how forest type 2, which corresponds with \textit{Lodgepole Pine} is the only type that does not present Elevation (1) as in its 10 most important features; this type of forest can be present from ocean shores to subalpine regions, therefore elevation is not a key feature when predicting this class. Finally, forest type 5 corresponds with \textit{Aspen} forest, which is typical of high plains and mountains; in this case, elevation (1) shows a negative coefficient, which is incorrect according to the previous forest description. However, this is the class that is never predicted by the NN, which can be seen in Table \ref{tab:confusionNN}. Therefore, the NN has not properly learned this forest type characteristics and NN2Poly allows to detect this problem by interpreting the coefficients and comparing them with the expectations from expert knowledge.

\section{Conclusions}

The proposed method, NN2Poly, obtains a polynomial representation that, under certain conditions, performs equivalently as a given densely-connected feed-forward neural network. This extends to deeper layers the work presented in \cite{moralaMathematicalFrameworkInform2021}, where it was solved for the case of a single hidden layer.

NN2Poly is first presented theoretically using an inductive approach that features Taylor expansions and the combinatorial properties of an MLP. The method only relies on the condition that the synaptic potentials fall in the acceptable range of the Taylor expansion, which can be achieved by imposing some weight constraints during the training phase and using certain activation functions like $\tanh()$. Furthermore, a practical and more efficient version is presented in Section \ref{section_practical_implementation}, which lightens some of the computational burdens by reducing the total order of the computed polynomial, and by simplifying the computation of multisets partitions. This shows promising results both for simulated and real tabular data in the context of regression and classification. However, there is still room for improving the efficiency of the method and increasing its application to even larger networks or higher dimension problems, as scalability still poses a challenge. Thus, advances in that regard should help expand the applications and simulations to higher-dimensional tabular data or non tabular data like images or audio.

Furthermore, future research should be conducted on improving some of the imposed constraints to obtain an efficient way to apply NN2Poly to a broader class of neural networks and activation functions. In this regard, the analysis of different regularization techniques is quite interesting, because depending on the selected technique, the effectiveness of NN2Poly may vary. Nevertheless, this analysis can be used to understand the impact of those constraints on the NNs through the obtained polynomials.

Several applications can be derived from this transformation from a NN into a polynomial. First of all, the number of parameters could be highly reduced, and this allows new interpretability and explainability techniques, as the coefficients of a polynomial have a classical statistical interpretation. Furthermore, the study of NN properties can benefit from this approach in several ways. For example, given that NN2Poly is able to generate a polynomial representation at every neuron that can be related to any other previous layer, this approach could be used to study partial internal representations learned by the NN. Moreover, these polynomials, as well as the output one, could be studied at different stages of the learning process to derive insightful conclusions. A rigorous study of these possibilities should emerge with the application of NN2Poly and the improvements in its efficiency.

\bibliographystyle{IEEEtran}
\bibliography{IEEEabrv,references/Main.bib}

\begin{thebibliography}{10}
\providecommand{\url}[1]{#1}
\csname url@samestyle\endcsname
\providecommand{\newblock}{\relax}
\providecommand{\bibinfo}[2]{#2}
\providecommand{\BIBentrySTDinterwordspacing}{\spaceskip=0pt\relax}
\providecommand{\BIBentryALTinterwordstretchfactor}{4}
\providecommand{\BIBentryALTinterwordspacing}{\spaceskip=\fontdimen2\font plus
\BIBentryALTinterwordstretchfactor\fontdimen3\font minus
  \fontdimen4\font\relax}
\providecommand{\BIBforeignlanguage}[2]{{%
\expandafter\ifx\csname l@#1\endcsname\relax
\typeout{** WARNING: IEEEtran.bst: No hyphenation pattern has been}%
\typeout{** loaded for the language `#1'. Using the pattern for}%
\typeout{** the default language instead.}%
\else
\language=\csname l@#1\endcsname
\fi
#2}}
\providecommand{\BIBdecl}{\relax}
\BIBdecl

\bibitem{lecunDeepLearning2015}
Y.~LeCun, Y.~Bengio, and G.~Hinton, ``Deep learning,'' \emph{Nature}, vol. 521,
  no. 7553, pp. 436--444, May 2015.

\bibitem{benitezAreArtificialNeural1997}
J.~Benitez, J.~Castro, and I.~Requena, ``Are artificial neural networks black
  boxes?'' \emph{IEEE Transactions on Neural Networks}, vol.~8, no.~5, pp.
  1156--1164, Sep. 1997.

\bibitem{shwartz-zivOpeningBlackBox2017}
R.~{Shwartz-Ziv} and N.~Tishby, ``Opening the {{Black Box}} of {{Deep Neural
  Networks}} via {{Information}},'' \emph{arXiv:1703.00810 [cs]}, Apr. 2017.

\bibitem{HIROSE199161}
Y.~Hirose, K.~Yamashita, and S.~Hijiya, ``Back-propagation algorithm which
  varies the number of hidden units,'' \emph{Neural Networks}, vol.~4, no.~1,
  pp. 61--66, 1991.

\bibitem{weymaereInitializationOptimizationMultilayer1994}
N.~Weymaere and J.-P. Martens, ``On the initialization and optimization of
  multilayer perceptrons,'' \emph{IEEE Transactions on Neural Networks},
  vol.~5, no.~5, pp. 738--751, 1994.

\bibitem{bengioPracticalRecommendationsGradientBased2012}
Y.~Bengio, ``Practical {{Recommendations}} for {{Gradient-Based Training}} of
  {{Deep Architectures}},'' in \emph{Neural {{Networks}}: {{Tricks}} of the
  {{Trade}}: {{Second Edition}}}, ser. Lecture {{Notes}} in {{Computer
  Science}}, G.~Montavon, G.~B. Orr, and K.-R. M{\"u}ller, Eds.\hskip 1em plus
  0.5em minus 0.4em\relax {Berlin, Heidelberg}: {Springer}, 2012, pp. 437--478.

\bibitem{saxeInformationBottleneckTheory2019}
A.~M. Saxe, Y.~Bansal, J.~Dapello, M.~Advani, A.~Kolchinsky, B.~D. Tracey, and
  D.~D. Cox, ``On the information bottleneck theory of deep learning,''
  \emph{Journal of Statistical Mechanics: Theory and Experiment}, vol. 2019,
  no.~12, p. 124020, Dec. 2019.

\bibitem{leeDeepNeuralNetworks2018}
J.~Lee, Y.~Bahri, R.~Novak, S.~S. Schoenholz, J.~Pennington, and
  J.~{Sohl-Dickstein}, ``Deep {{Neural Networks}} as {{Gaussian Processes}},''
  in \emph{International {{Conference}} on {{Learning Representations}}}, Feb.
  2018.

\bibitem{jacotNeuralTangentKernel2018}
A.~Jacot, F.~Gabriel, and C.~Hongler, ``Neural {{Tangent Kernel}}:
  {{Convergence}} and {{Generalization}} in {{Neural Networks}},'' in
  \emph{Advances in {{Neural Information Processing Systems}}}, vol.~31.\hskip
  1em plus 0.5em minus 0.4em\relax {Curran Associates, Inc.}, 2018.

\bibitem{agarwalNeuralAdditiveModels2021}
R.~Agarwal, L.~Melnick, N.~Frosst, X.~Zhang, B.~Lengerich, R.~Caruana, and
  G.~E. Hinton, ``Neural {{Additive Models}}: {{Interpretable Machine
  Learning}} with {{Neural Nets}},'' in \emph{Advances in {{Neural Information
  Processing Systems}}}, vol.~34.\hskip 1em plus 0.5em minus 0.4em\relax
  {Curran Associates, Inc.}, 2021, pp. 4699--4711.

\bibitem{lemhadriLassoNetNeuralNetworks2021}
I.~Lemhadri, F.~Ruan, and R.~Tibshirani, ``{{LassoNet}}: {{Neural Networks}}
  with {{Feature Sparsity}},'' in \emph{Proceedings of {{The}} 24th
  {{International Conference}} on {{Artificial Intelligence}} and
  {{Statistics}}}.\hskip 1em plus 0.5em minus 0.4em\relax {PMLR}, Mar. 2021,
  pp. 10--18.

\bibitem{chenAnalysisDesignMultiLayer1998}
M.-S. Chen, M.~Manry, K.~Yeung, V.~Devarajan, J.~Bredow, and D.~Levine,
  ``Analysis {{And Design Of The Multi-Layer Perceptron Using Polynomial Basis
  Functions}},'' Ph.D. dissertation, The University of Texas at Arlington, Dec.
  1998.

\bibitem{chenBackpropagationRepresentationTheorem1990}
M.-S. Chen and M.~Manry, ``Backpropagation representation theorem using power
  series,'' in \emph{1990 {{IJCNN International Joint Conference}} on {{Neural
  Networks}}}, Jun. 1990, pp. 643--648 vol.1.

\bibitem{chenConventionalModelingMultilayer1993}
------, ``Conventional modeling of the multilayer perceptron using polynomial
  basis functions,'' \emph{IEEE Transactions on Neural Networks}, vol.~4,
  no.~1, pp. 164--166, Jan. 1993.

\bibitem{cybenkoApproximationSuperpositionsSigmoidal1989}
G.~Cybenko, ``Approximation by superpositions of a sigmoidal function,''
  \emph{Mathematics of Control, Signals, and Systems}, vol.~2, no.~4, pp.
  303--314, Dec. 1989.

\bibitem{hornikMultilayerFeedforwardNetworks1989}
K.~Hornik, M.~Stinchcombe, and H.~White, ``Multilayer feedforward networks are
  universal approximators,'' \emph{Neural Networks}, vol.~2, no.~5, pp.
  359--366, Jan. 1989.

\bibitem{hornikApproximationCapabilitiesMultilayer1991}
K.~Hornik, ``Approximation capabilities of multilayer feedforward networks,''
  \emph{Neural Networks}, vol.~4, no.~2, pp. 251--257, 1991.

\bibitem{jayakumarMultiplicativeInteractionsWhere2019}
S.~M. Jayakumar, W.~M. Czarnecki, J.~Menick, J.~Schwarz, J.~Rae, S.~Osindero,
  Y.~W. Teh, T.~Harley, and R.~Pascanu, ``Multiplicative {{Interactions}} and
  {{Where}} to {{Find Them}},'' in \emph{International {{Conference}} on
  {{Learning Representations}}}, Sep. 2019.

\bibitem{chrysosDeepPolynomialNeural2022}
G.~G. Chrysos, S.~Moschoglou, G.~Bouritsas, J.~Deng, Y.~Panagakis, and
  S.~Zafeiriou, ``Deep {{Polynomial Neural Networks}},'' \emph{IEEE
  Transactions on Pattern Analysis and Machine Intelligence}, vol.~44, no.~8,
  pp. 4021--4034, Aug. 2022.

\bibitem{rendleFactorizationMachines2010}
S.~Rendle, ``Factorization {{Machines}},'' in \emph{2010 {{IEEE International
  Conference}} on {{Data Mining}}}, Dec. 2010, pp. 995--1000.

\bibitem{chrysosPolyGANHighOrderPolynomial2019}
G.~Chrysos, S.~Moschoglou, Y.~Panagakis, and S.~Zafeiriou, ``{{PolyGAN}}:
  {{High-Order Polynomial Generators}},'' \emph{arXiv:1908.06571}, Oct. 2019.

\bibitem{fanExpressivityTrainabilityQuadratic2023}
F.-L. Fan, M.~Li, F.~Wang, R.~Lai, and G.~Wang, ``On {{Expressivity}} and
  {{Trainability}} of {{Quadratic Networks}},'' \emph{arXiv:2110.06081}, Sep.
  2023.

\bibitem{chorariaSpectralBiasPolynomial2022}
M.~Choraria, L.~Dadi, G.~G. Chrysos, J.~Mairal, and V.~Cevher, ``The {{Spectral
  Bias}} of {{Polynomial Neural Networks}},'' in \emph{{{ICLR}} 2022 -
  {{International Conference}} on {{Learning Representations}}}, {France}, Apr.
  2022, pp. 1--30.

\bibitem{dubeyScalableInterpretabilityPolynomials2022}
A.~Dubey, F.~Radenovic, and D.~Mahajan, ``Scalable {{Interpretability}} via
  {{Polynomials}},'' in \emph{Advances in {{Neural Information Processing
  Systems}}}, vol.~35, Dec. 2022, pp. 36\,748--36\,761.

\bibitem{chengPolynomialRegressionAlternative2019}
X.~Cheng, B.~Khomtchouk, N.~Matloff, and P.~Mohanty, ``Polynomial {{Regression
  As}} an {{Alternative}} to {{Neural Nets}},'' \emph{arXiv:1806.06850 [cs,
  stat]}, Apr. 2019.

\bibitem{moralaMathematicalFrameworkInform2021}
P.~Morala, J.~A. Cifuentes, R.~E. Lillo, and I.~Ucar, ``Towards a mathematical
  framework to inform neural network modelling via polynomial regression,''
  \emph{Neural Networks}, vol. 142, pp. 57--72, Oct. 2021.

\bibitem{sanchiricoAMITENovelPolynomial2021}
M.~J. Sanchirico, X.~Jiao, and C.~Nataraj, ``{{AMITE}}: {{A Novel Polynomial
  Expansion}} for {{Analyzing Neural Network Nonlinearities}},'' \emph{IEEE
  Transactions on Neural Networks and Learning Systems}, pp. 1--13, 2021.

\bibitem{knuthArtComputerProgramming2005}
D.~E. Knuth, \emph{The Art of Computer Programming}.\hskip 1em plus 0.5em minus
  0.4em\relax {Upper Saddle River, NJ}: {Addison-Wesley}, 2005, vol. 4A:
  Combinatorial algorithms.

\bibitem{emschwillerNeuralNetworksPolynomial2020}
M.~Emschwiller, D.~Gamarnik, E.~C. K{\i}z{\i}lda{\u g}, and I.~Zadik, ``Neural
  {{Networks}} and {{Polynomial Regression}}. {{Demystifying}} the
  {{Overparametrization Phenomena}},'' \emph{arXiv:2003.10523 [cs, math,
  stat]}, Mar. 2020.

\bibitem{harrisonHedonicHousingPrices1978}
D.~Harrison and D.~L. Rubinfeld, ``Hedonic housing prices and the demand for
  clean air,'' \emph{Journal of Environmental Economics and Management},
  vol.~5, no.~1, pp. 81--102, Mar. 1978.

\bibitem{blackardComparativeAccuraciesArtificial1999}
J.~A. Blackard and D.~J. Dean, ``Comparative accuracies of artificial neural
  networks and discriminant analysis in predicting forest cover types from
  cartographic variables,'' \emph{Computers and Electronics in Agriculture},
  vol.~24, no.~3, pp. 131--151, Dec. 1999.

\bibitem{borisovDeepNeuralNetworks2022}
V.~Borisov, T.~Leemann, K.~Se{\ss}ler, J.~Haug, M.~Pawelczyk, and G.~Kasneci,
  ``Deep {{Neural Networks}} and {{Tabular Data}}: {{A Survey}},'' \emph{IEEE
  Transactions on Neural Networks and Learning Systems}, pp. 1--21, 2022.

\bibitem{kingmaAdamMethodStochastic2017}
D.~P. Kingma and J.~Ba, ``Adam: {{A Method}} for {{Stochastic Optimization}},''
  in \emph{International {{Conference}} on {{Learning Representations}}}, {San
  Diego, USA}, Jan. 2015.

\end{thebibliography}

\appendices

\section{Multiset partition algorithm}\label{appendix_multiset}

In this Appendix, details are presented on how to compute all possible partitions of a multiset $M$ (or multipartitions), associated to a combination of variables $\Vec{t}$, that are needed in $\pi(\Vec{t},Q,n)$ from equation Eq. \ref{eq_final_coeff_lemma}. Obtaining all of them can be done with Algorithm \ref{alg_knuth}, extracted from \cite{knuthArtComputerProgramming2005}. 

First, note that these multipartitions can be seen as sums of vectors with non-negative integer components, with each component representing each different element in the multiset, i.e., a representation of each partition using the notation $\Vec{t}$ from Eq \ref{eq_poly_vec_t}. Consider the example of multiset $M=\{1,1,2,3\}$. Then, some of its partitions are represented with vectors as shown in Table \ref{tab:multipartitions}.

With this representation, \cite{knuthArtComputerProgramming2005} presents a procedure to obtain all possible partitions of a given multiset (Algorithm \ref{alg_knuth}). Briefly, this algorithm uses tuples $(c,u,v)$ with $c$ representing the component of a multipartition vector representation, $u$ the remaining amount to be partitioned in component $c$ and $v$ the amount in the current partition. Then, the algorithm runs over all possible partitions, extracting the $v$ values for each of them.

Formally, these tuples are defined as follows. Define $(c,u,v)$ as 3-tuples representing a partition, where $c$ denotes the component index in the partition vector, which is equivalent to the integer representing the element in the multiset, $u > 0$ denotes the remaining in component $c$ that still needs to be partitioned, and $v$ represents the amount of the c component in the current partition, satisfying $0 \leq v \leq u$. Define $m$ as the number of elements in a multiset $\{n_1\cdot 1,\dots, n_m \cdot m\}$, and $n=n_1+\dots + n_m$. Then these 3-tuples are stored as arrays $(c_0,c_1,\dots, c_{nm})$, $(u_0,u_1,\dots, u_{nm})$, $(v_0,v_1,\dots, v_{nm})$ of length $m \cdot n$. An extra array $(f_0,f_1,\dots, f_{n})$ is used as a ``stack'' such that the $(l+1)$st vector of the partition is formed by the elements from $f_l$ to $f_{l+1}-1$ in the $c,u,v$ arrays. Then, using this arrays, Algorithm \ref{alg_knuth} obtains all possible partitions of a given multiset.

\begin{table}[H]
    \centering
    \caption{Examples of multipartitions and their vector representation.}
    \label{tab:multipartitions}
    \begin{tabular}{c|c}
        \hline
        Multipartition &  Vector representation\\
        \hline
        
        $\{1,1,2,3\}$ &
        $\begin{pmatrix}2, 1, 1\end{pmatrix}$\\
        
        $\{1\},\{1,2,3\}$ &
        $\begin{pmatrix}1,  0, 0\end{pmatrix}+
        \begin{pmatrix}1, 1, 1\end{pmatrix}$\\
        
        $\{1,1\},\{2,3\}$ &
        $\begin{pmatrix}2, 0, 0\end{pmatrix}+
        \begin{pmatrix}0, 1, 1\end{pmatrix}$\\
        
        $\{1\},\{2\},\{1,3\}$ &
        $\begin{pmatrix}1, 0, 0\end{pmatrix}+
        \begin{pmatrix}0, 1, 0\end{pmatrix}+
        \begin{pmatrix}1, 0, 1\end{pmatrix}$\\
        
        \hline
    \end{tabular}
\end{table}

\begin{algorithm}[H]
\small
\caption{\\Multipartitions in decreasing lexicographic order \cite{knuthArtComputerProgramming2005}}
\label{alg_knuth}
\begin{algorithmic}[1]
    \vspace{1mm}
    \phase{Initialize}
    \For{$0 \le j < m$}
        \State Set $c_j \leftarrow j+1$, $u_j \leftarrow v_j \leftarrow n_{j+1}$;
    \EndFor
    \State Set $f_0 \leftarrow a \leftarrow l \leftarrow 0$, $f_1 \leftarrow b \leftarrow m$;
    \phase{Subtract $v$ from $u$, $x$ will denote if $v$ has changed}
    \State Set $j \leftarrow a$, $k \leftarrow b$, $x \leftarrow 0$,
    \While{$j < b$}
        \State $u_k \leftarrow u_j - v_j$;
        \If{$u_k = 0$}
            \State Set $x \leftarrow 1$, $j \leftarrow j+1$;
        \ElsIf{$x=0$}
            \State $c_k \leftarrow c_j$, $v_k <- \min(v_j,u_k)$, $x \leftarrow [u_k < v_j]$, $ k \leftarrow k+1, j \leftarrow j+1$;
        \Else
            \State Set $c_{k} \leftarrow c_{j}, v_{k} \leftarrow u_{k}, k \leftarrow k+1, j \leftarrow j+1$;
        \EndIf
    \EndWhile
    \phase{Push if nonzero}
    \If{$k > b$}
    \State Set $a \leftarrow b$, $b \leftarrow k$, $l \leftarrow l+1$, $f_{l+1} \leftarrow b$;
    \State Return to \textbf{Part 2};
    \EndIf
    \phase{Visit a partition} 
    \For{$0 \le k \le l$}
        \For{$f_k \le j < f_{k+1}$}
            \State The vector contains $v_j$ in component $c_j$
        \EndFor
    \EndFor
    \phase{Decrease $v$}
    \State Set $j \leftarrow b - 1$;
    \While{$v_j=0$}
        \State Set $j \leftarrow j - 1$;
    \EndWhile
    \If{$j=a$ and $v_j=1$}
        \State Go to \textbf{Part 6};
    \Else
        \State Set $v_j \leftarrow v_{j}-1$;
        \For{$j < k < b$}
        \State Set $v_k \leftarrow u_k$;
        \EndFor
        \State Return to \textbf{Part 2};
    \EndIf
    \phase{Backtrack}
    \If{l=0}
    \State Terminate;
    \Else
    \State Set $l \leftarrow l - 1$, $b \leftarrow a$, $a \leftarrow f_l$;
    \State Return to \textbf{Part 5};
    \EndIf
\end{algorithmic}
\end{algorithm}

\section{Real data sets}\label{appendix_data sets}

Tables \ref{tab:boston} and \ref{tab:covertype} present the predictor variables descriptions for the Boston Housing and the Forest Covertype data sets. Table \ref{tab:covertype:response} contains the seven forest type identifiers for the response variable in the Covertype data set.

\begin{table}[ht]
    \centering
    \small
    \caption{Predictor variable identifiers and descriptions for the Boston Housing data set.}
    \label{tab:boston}
    \begin{tabular}{c|p{0.86\linewidth}}
    \hline
    \textbf{\#} & \textbf{Description}\\
    \hline
    1 & Per capita crime rate.\\
    2 & The proportion of residential land zoned for lots over 25,000 square feet.\\
    3 & The proportion of non-retail business acres per town.\\
    4 & Charles River dummy variable (= 1 if tract bounds river; 0 otherwise).\\
    5 & Nitric oxides concentration (parts per 10 million).\\
    6 & The average number of rooms per dwelling.\\
    7 & The proportion of owner-occupied units built before 1940.\\
    8 & Weighted distances to five Boston employment centers.\\
    9 & Index of accessibility to radial highways.\\
    10 & Full-value property-tax rate per \$10,000.\\
    11 & Pupil-teacher ratio by town.\\
    12 & 1000 * (Bk - 0.63) ** 2 where Bk is the proportion of Black people by town.\\
    13 & Percentage lower status of the population.\\
    \hline
    \end{tabular}
\end{table}

\begin{table}[ht]
    \centering
    \small
    \caption{Predictor variable identifiers and descriptions for the Covertype data set.}
    \label{tab:covertype}
    \begin{tabular}{c|p{0.86\linewidth}}
    \hline
    \textbf{\#} & \textbf{Description}\\
    \hline
    1 & Elevation in meters.\\
    2 & Aspect in degrees azimuth.\\
    3 & Slope in degrees.\\
    4 & Horizontal distance to nearest surface water features.\\
    5 & Vertical distance to nearest surface water features.\\
    6 & Horizontal distance to nearest roadway.\\
    7 & Hillshade index at 9am, summer solstice.\\
    8 & Hillshade index at noon, summer solstice.\\
    9 & Hillshade index at 3pm, summer solstice.\\
    10 & Horizontal distance to nearest wildfire ignition points.\\
    11 - 14 & Wilderness area designation (4 binary columns).\\
    15 - 54 & Soil Type designation (40 binary columns).\\
    \hline
    \end{tabular}
\end{table}

\begin{table}[ht]
    \centering
    \small
    \caption{Response variable identifiers and descriptions for the Covertype data set (forest types).}
    \label{tab:covertype:response}
    \begin{tabular}{c|p{0.86\linewidth}}
    \hline
    \textbf{\#} & \textbf{Description}\\
    \hline
    1 & Spruce/Fir.\\
    2 & Lodgepole Pine.\\
    3 & Ponderosa Pine.\\
    4 & Cottonwood/Willow.\\
    5 & Aspen.\\
    6 & Douglas-fir.\\
    7 & Krummholz.\\
    \hline
    \end{tabular}
\end{table}

\end{document}